\documentclass[nohyperref]{article}

\usepackage{microtype}
\usepackage{graphicx}
\usepackage{subfigure}
\usepackage{booktabs} 

\PassOptionsToPackage{hyphens}{url}
\usepackage{hyperref}

\usepackage[accepted]{icml2022}

\usepackage{amsmath}
\usepackage{amssymb}
\usepackage{mathtools}
\usepackage{amsthm}
\usepackage{mathrsfs}
\usepackage{thmtools}
\usepackage{xspace}
\usepackage[mathcal]{euscript}

\usepackage[capitalize,noabbrev]{cleveref}

\theoremstyle{plain}

\theoremstyle{definition}

\theoremstyle{remark}

\makeatletter
\DeclareRobustCommand\onedot{\futurelet\@let@token\@onedot}
\def\@onedot{\ifx\@let@token.\else.\null\fi\xspace}

\def\eg{\emph{e.g}\onedot} 
\def\ie{\emph{i.e}\onedot} 
\def\cf{\emph{cf}\onedot}

\makeatother

\definecolor{mycolor}{RGB}{219,90,107}

\DeclareMathOperator*{\argmin}{arg\,min}

\usepackage[textsize=tiny]{todonotes}

\icmltitlerunning{The Geometry of Robust Value Functions}

\begin{document}

\twocolumn[
\icmltitle{The Geometry of Robust Value Functions}



\icmlsetsymbol{equal}{*}

\begin{icmlauthorlist}
\icmlauthor{Kaixin Wang}{nus-ids}
\icmlauthor{Navdeep Kumar}{technion-ece}
\icmlauthor{Kuangqi Zhou}{nus-ece}
\icmlauthor{Bryan Hooi}{nus-ids,nus-soc}
\icmlauthor{Jiashi Feng}{byte}
\icmlauthor{Shie Mannor}{technion-ece,nvidia}
\end{icmlauthorlist}

\icmlaffiliation{nus-ids}{Institute of Data Science, National University of Singapore, Singapore}
\icmlaffiliation{nus-ece}{Department of Electrical and Computer Engineering, National University of Singapore, Singapore}
\icmlaffiliation{nus-soc}{School of Computing, National University of Singapore, Singapore}
\icmlaffiliation{technion-ece}{Electrical and Computer Engineering, Technion, Haifa, Israel}
\icmlaffiliation{byte}{ByteDance, Singapore}
\icmlaffiliation{nvidia}{NVIDIA Research, Haifa, Israel}

\icmlcorrespondingauthor{Kaixin Wang}{kaixin.wang@u.nus.edu}

\icmlkeywords{Machine Learning, ICML}

\vskip 0.3in
]



\printAffiliationsAndNotice{}  


\begin{abstract}
The space of value functions is a fundamental concept in reinforcement learning.
Characterizing its geometric properties may provide insights for optimization and representation.
Existing works mainly focus on the value space for Markov Decision Processes (MDPs).
In this paper, we study the geometry of the robust value space for the more general Robust MDPs (RMDPs) setting, where transition uncertainties are considered.
Specifically, since we find it hard to directly adapt prior approaches to RMDPs, we start with revisiting the non-robust case, and introduce a new perspective that enables us to characterize both the non-robust and robust value space in a similar fashion.
The key of this perspective is to decompose the value space, in a state-wise manner, into unions of hypersurfaces.
Through our analysis, we show that the robust value space is determined by a set of \emph{conic hypersurfaces}, each of which contains the robust values of all policies that agree on one state. 
Furthermore, we find that taking only extreme points in the uncertainty set is sufficient to determine the robust value space.
Finally, we discuss some other aspects about the robust value space, including its non-convexity and policy agreement on multiple states.
\end{abstract}

\section{Introduction}
\label{intro}

\begin{figure}[t]
    \centering
    \includegraphics[width=\linewidth]{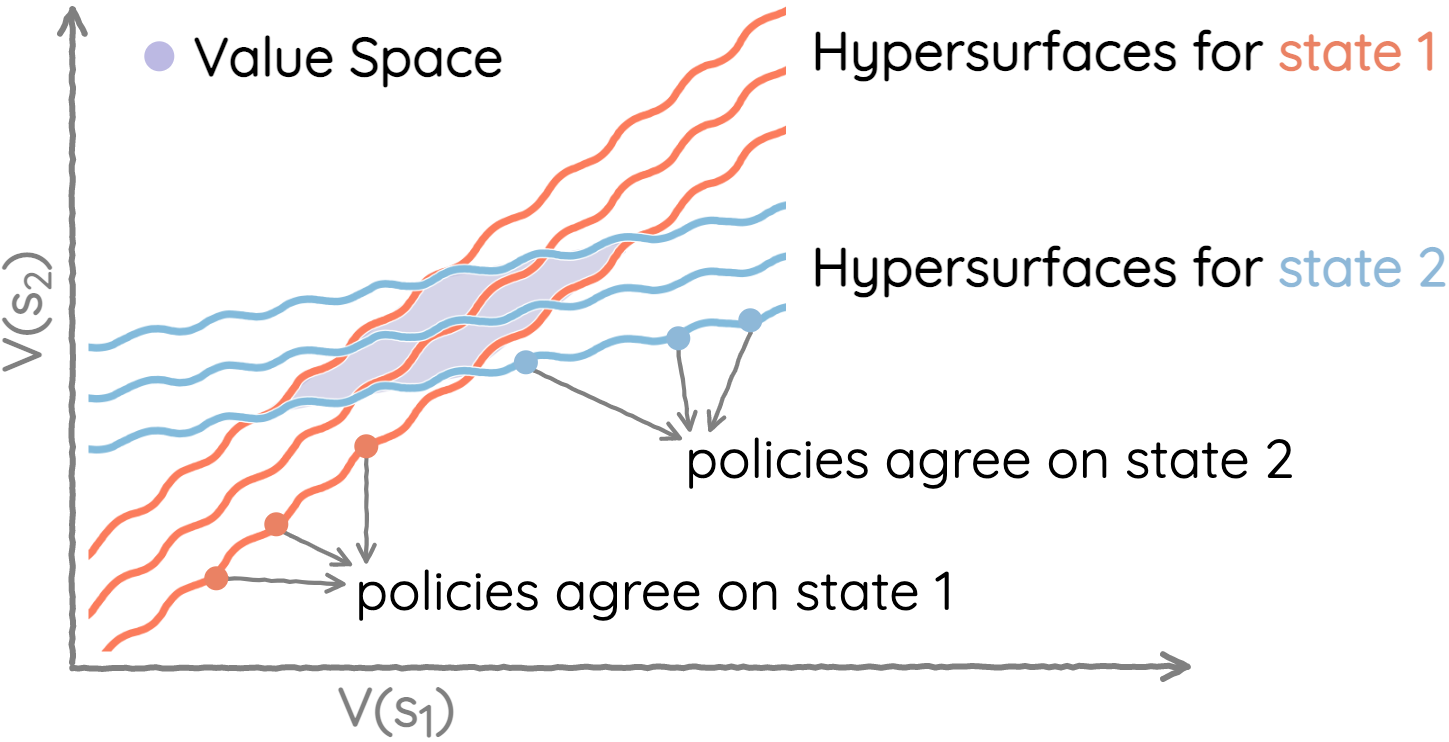}
    \caption{The value space can be decomposed in a state-wise manner as an intersection of unions of hypersurfaces. Each union corresponds to a state and each hypersurface contains the value functions of policies agreeing on that state.}
     \label{fig:opening}
\end{figure}

The space of value functions for stationary policies is a central concept in Reinforcement Learning (RL), since many RL algorithms are essentially navigating this space to find an optimal policy that maximizes the value function, such as policy gradient~\cite{sutton2000policy}, policy iteration~\cite{howard1960dynamic} and evolutionary strategies~\cite{de2005tutorial}.
Characterizing the geometric properties for the space of the value function (\ie, the value space) would offer insights for RL research. A recent work~\cite{dadashi2019value} shows that the value space for Markov Decision Processes (MDPs) is a possibly non-convex polytope, which inspires new methods in representation learning in RL~\cite{bellemare2019geometric, dabney2021value}.

Compared to MDPs, Robust MDPs (RMDPs) are more general, since they do not assume that the transition dynamics are known exactly but instead may take any value from a given uncertainty set~\cite{xu2006robustness, iyengar2005robust,nilim2005robust,wiesemann2013robust}. This makes RMDPs more suitable for real-world problems where parameters may not be precisely given. Therefore, characterizing the geometric properties of the value space for RMDPs (\ie, robust value space) is of interest.

However, we find it hard to directly adapt the prior approach~\cite{dadashi2019value} from MDPs to RMDPs. Their method builds upon on a key theorem (the Line Theorem), but we find it difficult to prove a robust counterpart of this theorem (see more discussions in Section~\ref{sec:line-theorem}).

In this work, we introduce a new perspective for investigating the geometry of the space of value functions.
Specifically, we start with revisiting the non-robust case due to its simplicity.
By decomposing the value space in a state-wise manner (as illustrated in Figure~\ref{fig:opening}), we can give an explicit form about the value function polytope.

With this decomposition-based perspective, we show that the robust value space is determined by a set of \emph{conic hypersurfaces}, each of which contains the robust value functions for policies that agree on one state.
Furthermore, from a geometric perspective, we show that the robust value space can be fully determined by a subset of the uncertainty set, which composes of extreme points of the uncertainty set.
As a result, for polyhedral uncertainty set such as $\ell_1$-ball and $\ell_\infty$-ball~\cite{ho2018fast, ho2021partial, behzadian2021fast}, we can replace the infinite uncertainty set with a finite active uncertainty subset, without losing any useful information for policy optimization. Finally, we discuss some other aspects about the robust value space, including policy agreement on more than one state, the non-convexity of the robust value space, and why it is difficult to obtain a Line Theorem for RMDPs.

All proofs and the specifics of MDPs and RMDPs used for illustration can be found in Appendix.

\section{Preliminaries}
\label{sec:prelim}

We introduce backgrounds for MDPs in
Section~\ref{sec:bg-mdp} and for RMDPs in Section~\ref{sec:bg-rmdp}. Importantly, Section~\ref{sec:bg-vfspace} sets up some essential concepts and notations for studying the value space, which will be frequently used in the rest of paper.

\emph{Notations.}
We use $\mathbf{1}$ and $\mathbf{0}$ to denote vectors of all ones and all zeros respectively, and their sizes can be inferred from the context. For vectors and matrices, $<$, $\le$, $>$ and $\ge$ denote element-wise comparisons. Calligraphic letters such as $\mathcal{P}$ are mainly for sets. For an index set $\mathcal{Z}=\{1,\cdots,k\}$, $(x_i)_{i\in\mathcal{Z}}$ denotes a vector $(x_1, x_2, \cdots, x_k)$ if $x_i$ is a scalar, or a matrix $(x_1, x_2, \cdots, x_k)^\top$ if $x_i$ is a vector. $\Delta_\mathcal{U}$ is used to denote the space of probability distributions over a set $\mathcal{U}$. For a non-empty set $\mathcal{U}$, we denote its \emph{polar cone} as $\mathcal{U}^*$\cite{bertsekas2009convex}, given by
\begin{equation}
    \mathcal{U}^* \coloneqq \{y \mid \langle y,x\rangle \le 0, \forall\,x\in\mathcal{U}\}.
\end{equation}
We use $\mathbf{conv}(\cdot)$ to denote the convex hull of a set, and $\mathbf{ext}(\cdot)$ to denote the set of extreme points of a non-empty convex set.

\subsection{Markov Decision Processes}
\label{sec:bg-mdp}

We consider an MDP $(\mathcal{S}, \mathcal{A}, P, r, \gamma, p_0)$ with a finite state set $\mathcal{S}$ and a finite action set $\mathcal{A}$.
The number of states $\lvert\mathcal{S}\rvert$ and the number of actions $\lvert\mathcal{A}\rvert$ are denoted with $S$ and $A$, respectively.
The initial state is generated according to the $p_0\in\Delta_\mathcal{S}$.
We use $P_{s,a}\in\Delta_\mathcal{S}$ to specify the probabilities of transiting to new states when taking action $a$ in state $s$, and employ $P\coloneqq(P_{s,a})_{s\in\mathcal{S},a\in\mathcal{A}}\in(\Delta_\mathcal{S})^{\mathcal{S}\times\mathcal{A}}$ as a condensed notation.
An immediate reward $r_{s,a}\in\mathbb{R}$ is given after taking action $a$ in state
$s$, and similarly $r\coloneqq(r_{s,a})_{s\in\mathcal{S},a\in\mathcal{A}}\in\mathbb{R}^{\mathcal{S}\times\mathcal{A}}$ is a condensed notation.
$\gamma\in[0,1)$ is the discount factor.
In addition, we also define $P_s\coloneqq(P_{s,a})_{a\in\mathcal{A}}\in(\Delta_\mathcal{S})^{\mathcal{A}}$ and $r_s\coloneqq(r_{s,a})_{a\in\mathcal{A}}\in\mathbb{R}^{\mathcal{A}}$.

A stationary stochastic policy $\pi\coloneqq(\pi_{s,a})_{s\in\mathcal{S},a\in\mathcal{A}}\in(\Delta_\mathcal{A})^{\mathcal{S}}$ specifies a decision making strategy, where $\pi_{s,a}\in[0,1]$ is the probability of taking some action $a$ in current state $s$. We denote $\pi_s\coloneqq(\pi_{s,a})_{a\in\mathcal{A}}\in\Delta_\mathcal{A}$ as the probability vector over actions. In particular, we use $d_{s,a}\in\Delta_\mathcal{A}$ to represent a deterministic $\pi_s$ that is all-zero except $\pi_{s,a}=1$.

Under a given policy $\pi$, we define the state-to-state transition probability as
\begin{equation}
    \begin{aligned}
    P^\pi &\coloneqq (P^{\pi_s})_{s\in\mathcal{S}}\in(\Delta_\mathcal{S})^\mathcal{S}, \quad \textrm{where} \\
    P^{\pi_s} &\coloneqq P_s\pi_s= \sum_{a\in\mathcal{A}}\pi_{s,a}P_{s,a}\in\Delta_\mathcal{S}.
\end{aligned}
\end{equation}

The reward function under this policy is defined as
\begin{equation}
\begin{aligned}
    r^\pi &\coloneqq (r^{\pi_s})_{s\in\mathcal{S}}\in\mathbb{R}^\mathcal{S}, \quad \textrm{where} \\
    r^{\pi_s} &\coloneqq r_s^\top\pi_s = \sum_{a\in\mathcal{A}}\pi_{s,a}r_{s,a} \in\mathbb{R}.
\end{aligned}
\end{equation}
The value $V^{\pi,P}\in\mathbb{R}^\mathcal{S}$ is defined to be the expected cumulative reward from starting in a state and acting according to the policy $\pi$ under transition dynamic $P$:
\begin{equation}
    V^{\pi,P}(s) \coloneqq \mathbb{E}_{P^\pi} \left[\sum_{t=0}^\infty \gamma^t r_{s_t, a_t}\mid s_0=s\right].
\end{equation}

\subsection{Robust Markov Decision Processes}
\label{sec:bg-rmdp}

Robust Markov Decision Processes (RMDPs) generalize MDPs in that the uncertainty in the transition dynamic $P$ is considered~\cite{iyengar2005robust,nilim2005robust,wiesemann2013robust}. In an RMDP, the transition dynamic $P$ is chosen adversarially from an uncertainty set $\mathcal{P}\subseteq(\Delta_\mathcal{S})^{\mathcal{S}\times \mathcal{A}}$. We assume throughout the paper that the set $\mathcal{P}$ is compact. The robust value function for a policy $\pi$ and the optimal robust value function are defined as
\begin{align}
    V^{\pi,\mathcal{P}}(s) &\coloneqq \min_{P\in\mathcal{P}} V^{\pi, P}(s),\\ V^{\star,\mathcal{P}}(s) &\coloneqq \max_{\pi\in\Pi} V^{\pi,\mathcal{P}}(s).
\end{align}
Both the policy evaluation and policy improvement problems are intractable for generic $\mathcal{P}$~\cite{wiesemann2013robust}. However, they become tractable when certain independence assumptions about $\mathcal{P}$ are made. Two common assumptions are \emph{$(s,a)$-rectangularity}~\cite{iyengar2005robust, nilim2005robust} and \emph{$s$-rectangularity}~\cite{wiesemann2013robust}, which we will use in this paper. The $(s,a)$-rectangularity assumes that the adversarial nature selects the worst transition probabilities independently for each state and action. Under $(s,a)$-rectangularity, the uncertainty set $\mathcal{P}$ can be factorized into $\mathcal{P}_{s,a}\subseteq\Delta_\mathcal{S}$ for each state-action pair, \ie,
\begin{equation}
    \mathcal{P} = \{P\mid P_{s,a}\in \mathcal{P}_{s,a}, \forall\,s\in\mathcal{S}, \forall\,a\in\mathcal{A}\},
\end{equation}
or in short $\mathcal{P}=\underset{(s,a)\in\mathcal{S}\times\mathcal{A}}{\times}\mathcal{P}_{s,a}$ where $\times$ denotes Cartesian product. The \emph{$s$-rectangularity} is less restrictive and assumes the adversarial nature selects the worst transition probabilities independently for each state. Under $s$-rectangularity, the uncertainty set $\mathcal{P}$ can be factorized into $\mathcal{P}_s\subseteq(\Delta_\mathcal{S})^\mathcal{A}$ for each state, \ie,
\begin{equation}
    \mathcal{P} = \{P\mid P_s\in \mathcal{P}_s, \forall\,s\in\mathcal{S}\},
\end{equation}
or in short $\mathcal{P}=\underset{s\in\mathcal{S}}{\times}\mathcal{P}_s$.
Note that $(s,a)$-rectangularity is a special case of $s$-rectangularity. Below we present a restatement of the remark in~\cite{ho2021partial} that the optimal policy for the robust policy evaluation MDP is deterministic. This restatement will be used later. Under $s$-rectangularity, we have for any $\pi$,
\begin{equation}
    \begin{aligned}
    \exists\,P\in\mathcal{P}\quad \textrm{s.t.} \quad V^{\pi,P}(s) = V^{\pi,\mathcal{P}}(s),\, \forall s\in\mathcal{S}.
\end{aligned}
\label{eqn:s-rect-exist}
\end{equation}

\subsection{The Space of Value Functions}
\label{sec:bg-vfspace}

The space of value functions (or value space in short) is the set of value functions for all stationary policies. We use $f_P$ and $f_{\mathcal{P}}$ to respectively represent the mapping between a set of policies and their non-robust and robust value functions, \ie,
\begin{align}
    f_P(U) &\coloneqq \{V^{\pi,P}\mid \pi\in U\}, \\
    f_{\mathcal{P}}(U) &\coloneqq \{V^{\pi,\mathcal{P}}\mid \pi\in U\}.
\end{align}
The set of all stationary stochastic policies is denoted as $\Pi=(\Delta_\mathcal{A})^\mathcal{S}$. Then, the non-robust value space for a transition dynamic $P$ and the robust value space for an uncertainty set $\mathcal{P}$ can be respectively expressed as
\begin{align}
    \mathcal{V}^P &\coloneqq f_P(\Pi), \label{eqn:def-vf-space}
    \\ \mathcal{V}^\mathcal{P} &\coloneqq f_{\mathcal{P}}(\Pi).
    \label{eqn:def-robsut-vf-space}
\end{align}
We then introduce some notations that will be frequently used later. We use $Y^{\pi_s}$ to denote the set of policies that agree with $\pi$ on $s$, \ie,
\begin{equation}
    Y^{\pi_s} \coloneqq \{\pi'\mid \pi'_s = \pi_s\}.
\end{equation}
Note that policy agreement on state $s$ does not imply disagreement on other states. Thus, $\pi$ itself is also in $Y^{\pi_s}$. The row of the matrix $I - \gamma P^\pi$ corresponds to state $s$ is denoted as $L^{\pi_s, P_s}$, \ie,
\begin{equation}
    L^{\pi_s, P_s} \coloneqq \mathbf{e}_s - \gamma P^{\pi_s} = \mathbf{e}_s - \gamma P_s \pi_s
\end{equation}
where $\mathbf{e}_s\in\mathbb{R}^\mathcal{S}$ is an all-zero vector except the entry corresponding to $s$ being 1.

\section{The Value Function Polytope Revisited}
\label{sec:revisit}

\begin{figure}[t]
    \centering
    \includegraphics[width=\linewidth]{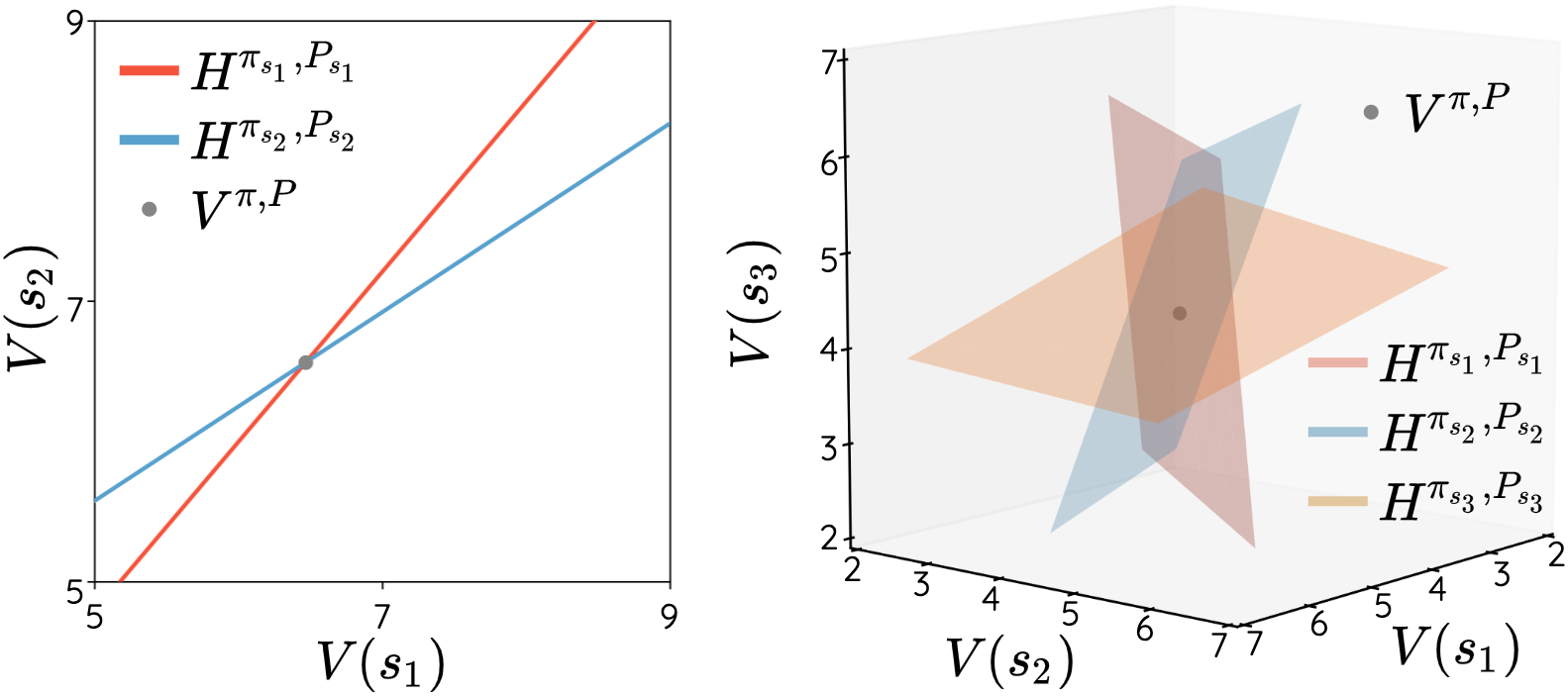}
    \caption{Hyperplanes $H^{\pi_s, P_s}$ corresponding to different $s$ intersect at the value function $V^{\pi,P}$.}
     \label{fig:single-v}
\end{figure}

In this section, we revisit the non-robust value space from a new perspective, where the value space is decomposed in a state-wise manner. 
This perspective enables us to characterize the polytope shape of the value space in a more straightforward way, and leads to an explicit form of the value polytope.

Our first step is to connect a single value function $V^{\pi,P}$ to a set of hyperplanes, each of which can be expressed as:
\begin{equation}
\label{}
    H^{\pi_s, P_s} \coloneqq \{\mathbf{x}\in\mathbb{R}^\mathcal{S}\mid \langle \mathbf{x}, L^{\pi_s,P_s} \rangle = r^{\pi_s} \}.
\end{equation}
As shown in Lemma~3 in \cite{dadashi2019value}, the value functions $f_P(Y^{\pi_s})$ lie in the hyperplane $H^{\pi_s, P_s}$.

Specifically, since $\pi\in Y^{\pi_s}$, we know every hyperplane $H^{\pi_s, P_s}$ passes through $V^{\pi,P}$ (see examples in Figure~\ref{fig:single-v}). 
The following lemma states that this intersecting point is unique.
\begin{restatable}{lemma}{singlev}
    Consider a policy $\pi$ and a transition dynamic $P$, we have
\begin{equation}
    \{V^{\pi,P}\} = \bigcap_{s\in\mathcal{S}} H^{\pi_s, P_s}
\end{equation}
\label{lem:single-v}
\end{restatable}
Lemma~\ref{lem:single-v} bridges between a single value function and the intersection of $S$ different hyperplanes, each of which corresponds to a state $s$.
Then, by definition (Eqn.~\eqref{eqn:def-vf-space}), we can obtain the value space by taking the union over all $\pi\in\Pi$, \ie,
\begin{equation}
\label{eqn:vf-space-union}
    \mathcal{V}^{P} = \bigcup_{\pi\in\Pi} \bigcap_{s\in\mathcal{S}} H^{\pi_s,P_s},
\end{equation}
as illustrated in Figure~\ref{fig:vf-space}(a).

From Eqn.~\eqref{eqn:vf-space-union}, we observe that the value space $\mathcal{V}^P$ can also be expressed from an alternative perspective (as shown in Figure~\ref{fig:vf-space}(b)):
1) for each state $s\in\mathcal{S}$, taking the union of all hyperplanes corresponding to different $\pi_s\in\Delta_\mathcal{A}$; 
2) taking the intersection of the unions obtained in previous step. 
The following lemma formalizes this perspective.
\begin{figure*}[t]
    \centering
    \includegraphics[width=\linewidth]{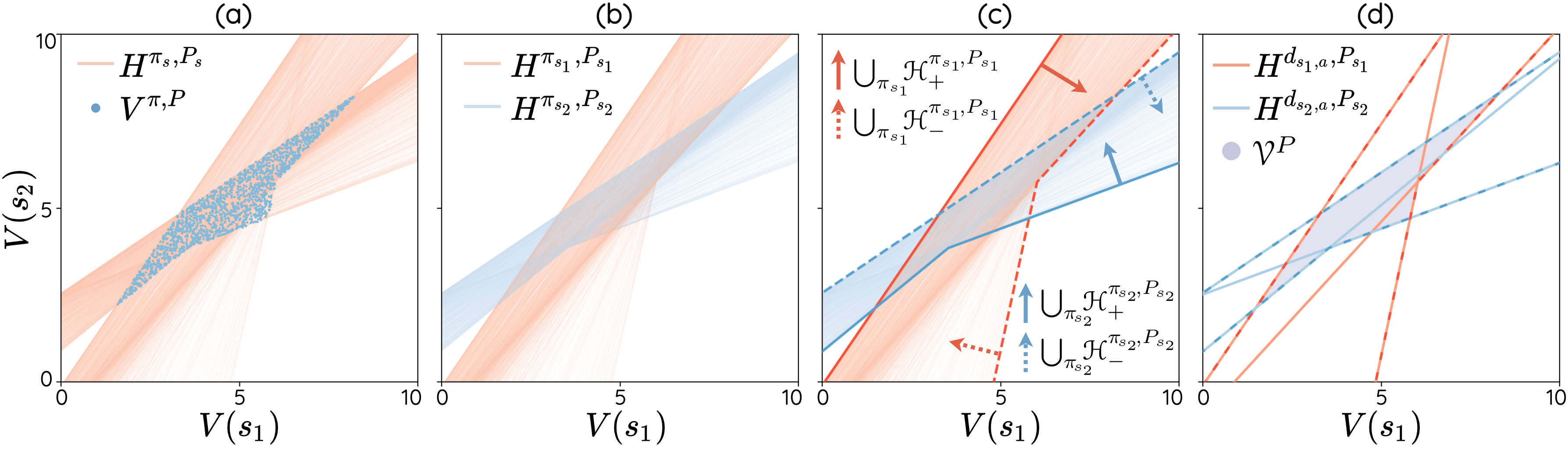}
    \caption{Visualization of the value functions for a 2-state 3-action MDP.
    \textbf{(a)} For each policy $\pi$, we plot the value function $V^{\pi,P}$ and the corresponding hyperplanes $H^{\pi_s, P_s}$ intersecting at $V^{\pi,P}$. 
    \textbf{(b)} For each policy $\pi$, the hyperplanes $H^{\pi_s, P_s}$ intersecting at $V^{\pi,P}$ are plotted in different colors for different states.
    \textbf{(c)} For each state $s$, the union of $\mathcal{H}^{\pi_s,P_s}_{+}$ and the union of $\mathcal{H}^{\pi_s,P_s}_{-}$ over all $\pi_s\in\Delta_\mathcal{A}$ are highlighted respectively.
    \textbf{(d)} For each state $s$, the hyperplanes $H^{d_{s,a},P_s}$ for different actions $a$ are plotted. The union of $\mathcal{H}_{+}^{d_{s,a},P_s}$ and the union of $\mathcal{H}_{-}^{d_{s,a},P_s}$ over all actions $a\in\mathcal{A}$ are highlighted as dashed. The entire value space $\mathcal{V}^P$ is visualized as the purple region.
    }
     \label{fig:vf-space}
\end{figure*}

\begin{restatable}{lemma}{vfspace}
    Consider a transition dynamic $P$, the value space $\mathcal{V}^P$ can be represented as
\begin{equation}
  \mathcal{V}^{P} = \bigcup_{\pi\in\Pi} \bigcap_{s\in\mathcal{S}}\,  H^{\pi_s,P_s} = \bigcap_{s\in\mathcal{S}}\, \bigcup_{\pi_s\in\Delta_\mathcal{A}} H^{\pi_s,P_s}.
\end{equation}
\label{lem:vf-space}
\end{restatable}
As suggested in Lemma~\ref{lem:vf-space}, the core of this perspective is to decompose the value space in a state-wise manner.
In this way, to study the whole value space, we only need to focus on the union of hyperplanes corresponding to one state.

Specifically, let us denote the two closed half-spaces determined by the hyperplane $H^{\pi_{s},P_s}$ as
\begin{equation}
    \begin{aligned}
    \mathcal{H}^{\pi_{s},P_s}_{+} &\coloneqq \{\mathbf{x}\in\mathbb{R}^\mathcal{S}\mid \langle \mathbf{x}, L^{\pi_s,P_s} \rangle \ge r^{\pi_s} \}, \\
    \mathcal{H}^{\pi_{s},P_s}_{-} &\coloneqq \{\mathbf{x}\in\mathbb{R}^\mathcal{S}\mid \langle \mathbf{x}, L^{\pi_s,P_s} \rangle \le r^{\pi_s} \}.
\end{aligned}
\end{equation}
Then the value space can be expressed in terms of the half-spaces:
\begin{equation}
\label{eqn:vfspace}
    \mathcal{V}_P = \bigcap_{s\in\mathcal{S}}\, \bigcup_{\pi_s\in\Delta_\mathcal{A}} \mathcal{H}^{\pi_s,P_s}_{+} \cap \mathcal{H}^{\pi_s,P_s}_{-}.
\end{equation}
Recall that in~\cite{dadashi2019value} a convex polyhedron is defined as a finite intersection of half-spaces, and a polytope is a bounded finite union of convex polyhedra. So our goal is to get rid of this infinite union $\bigcup_{\pi_s\in\Delta_\mathcal{A}}$.

To this end, we first replace the inner union in Eqn.~\eqref{eqn:vfspace} with an intersection of two unions, as illustrated in Figure~\ref{fig:vf-space}(c) and formally stated in the following lemma.
\begin{restatable}{lemma}{vfspacetwo}
    Consider a policy $\pi$ and a transition dynamic $P$, we have for all states $s\in\mathcal{S}$,
\begin{multline}
  \bigcup_{\pi_s\in\Delta_{\mathcal{A}}} \mathcal{H}^{\pi_s,P_s}_{+}\,\cap\, \mathcal{H}^{\pi_s,P_s}_{-}
  = \\
  \left[
  \bigcup_{\pi_s\in\Delta_{\mathcal{A}}} \mathcal{H}^{\pi_s,P_s}_{+}\right] \cap \left[\bigcup_{\pi_s\in\Delta_{\mathcal{A}}} \mathcal{H}^{\pi_s,P_s}_{-}\right].
\end{multline}
\label{lem:vf-space-2}
\end{restatable}
Although these two unions are still taken over infinite set $\Delta_\mathcal{A}$, the following Lemma~\ref{lem:vf-polyhedron} shows that they actually coincide with the finite unions of half-spaces that correspond to $d_{s,a}$ (\ie, deterministic $\pi_s$). We can get an intuition by comparing Figure~\ref{fig:vf-space}(c) and Figure~\ref{fig:vf-space}(d).
\begin{restatable}{lemma}{vfpolyhedronlem}
    Consider a policy $\pi$ and a transition dynamic $P$, we have for all states $s\in\mathcal{S}$,
\begin{equation}
  \bigcup_{\pi_s\in\Delta_{\mathcal{A}}} \mathcal{H}^{\pi_s,P_s}_{\delta} =
  \bigcup_{a\in\mathcal{A}} \mathcal{H}^{d_{s,a},P_s}_{\delta}, \quad \forall\,\delta\in\{+,-\}.
\end{equation}
\label{lem:vf-polyhedron}
\end{restatable}
Finally, putting everything together, we are able to represent the value space with finite union and intersection operations on half-spaces, as stated in Theorem~\ref{thm:vf-polyhedron} and illustrated in Figure~\ref{fig:vf-space}(d). Using the distributive law of sets, we can see that the value space $\mathcal{V}^P$ immediately satisfies the definition of polyhedron. Since $\mathcal{V}^P$ is bounded, we can conclude that $\mathcal{V}^P$ is a polytope.
\begin{restatable}{theorem}{vfpolyhedron}
Consider a transition dynamic $P$, the value space $\mathcal{V}^P$ can be represented as
\begin{equation}
    \begin{aligned}
    \mathcal{V}^{P} & = \bigcap_{s\in \mathcal{S}}\left[ \left[\bigcup_{a\in\mathcal{A}} \mathcal{H}^{d_{s,a},P_s}_{+}\right] \cap \left[\bigcup_{a\in\mathcal{A}} \mathcal{H}^{d_{s,a},P_s}_{-}\right] \right] \\
    & = \bigcup_{\mathbf{a}\in\mathcal{A}^{\mathcal{S}}} \bigcup_{\mathbf{a}'\in\mathcal{A}^{\mathcal{S}}} \bigcap_{s\in\mathcal{S}}
    \left[\mathcal{H}^{d_{s,\mathbf{a}_s},P_s}_{+} \cap
    \mathcal{H}^{d_{s,\mathbf{a}'_s},P_s}_{-}
    \right]
\end{aligned}
\end{equation}
where $\mathbf{a}=(\mathbf{a}_s)_{s\in\mathcal{S}}$, $\mathbf{a}'=(\mathbf{a}'_s)_{s\in\mathcal{S}}$, and $\mathbf{a}_s, \mathbf{a}'_s\in\mathcal{A}$.
\label{thm:vf-polyhedron}
\end{restatable}
Compared to the prior approach~\cite{dadashi2019value}, our work gives an explicit form of the value function polytope, showing how the value polytope is formed (\cf the proof of Proposition~1 in~\cite{dadashi2019value}).

\begin{figure}[t]
    \centering
    \includegraphics[width=\linewidth]{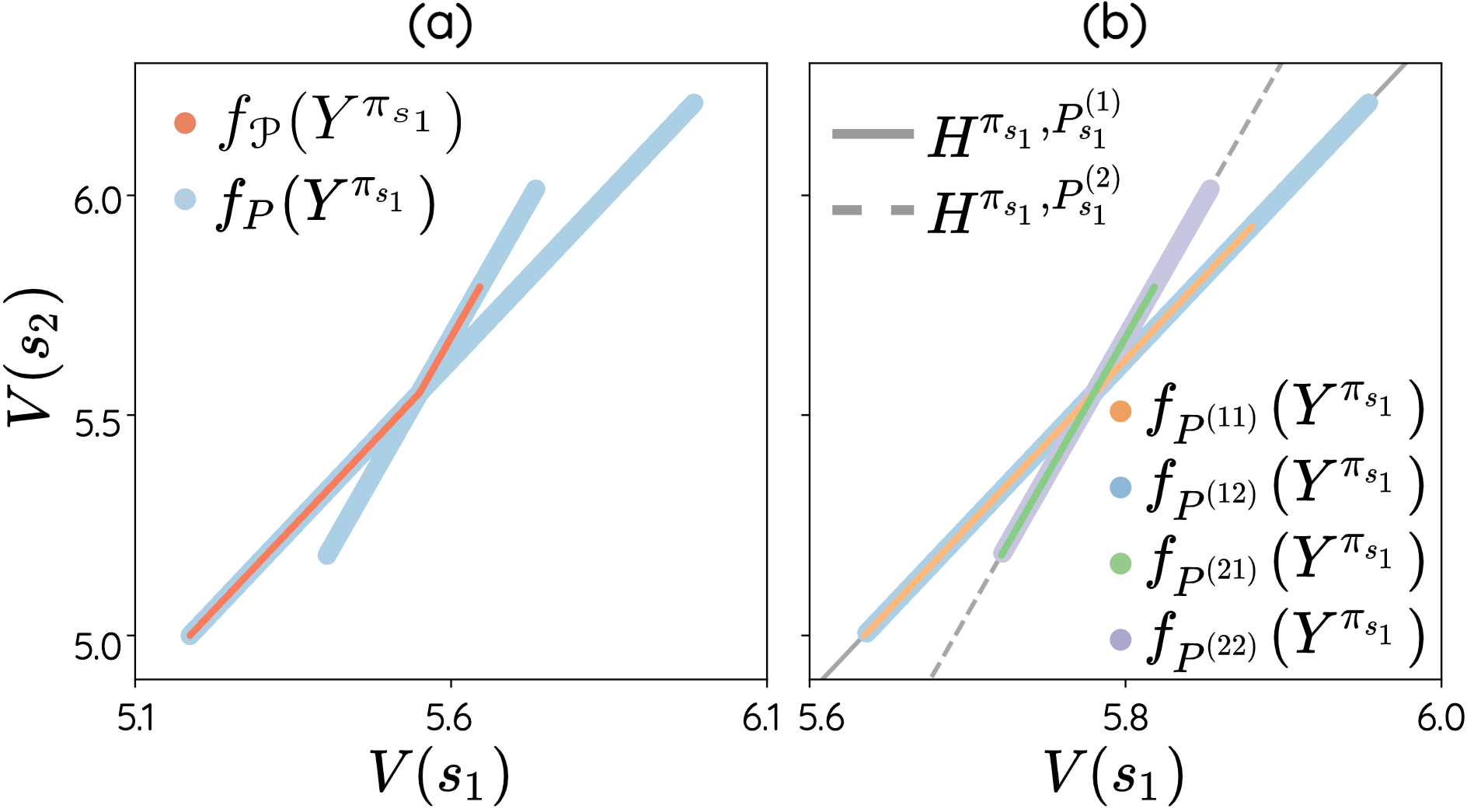}
    \caption{Visualization of the robust value functions for a 2-state 2-action RMDP with an $s$-rectangular uncertainty set. We consider $\mathcal{P}=\mathcal{P}_{s_1}\times\mathcal{P}_{s_2}$ with $\mathcal{P}_{s_1}=\{P_{s_1}^{(1)},P_{s_1}^{(2)}\}$ and $\mathcal{P}_{s_2}=\{P_{s_2}^{(1)},P_{s_2}^{(2)}\}$. Denote $P^{(ij)}\in\mathcal{P}$ such that $P^{(ij)}_{s_1}=P_{s_1}^{(i)}$ and $P^{(ij)}_{s_2}=P_{s_2}^{(j)}$. We plot with different widths to differentiate overlapping lines. \textbf{(a)} For the same set of policies $Y^{\pi_{s_1}}$, the set of non-robust value functions $f_P(Y^{\pi_{s_1}})$ for different $P\in\mathcal{P}$ and the set of robust value functions $f_\mathcal{P}(Y^{\pi_{s_1}})$ are plotted. \textbf{(b)} For different $P\in\mathcal{P}$, $f_P(Y^{\pi_{s_1}})$ are highlighted in different colors. The hyperplanes corresponding to different $P_{s_1}\in\mathcal{P}_{s_1}$ are plotted.}
     \label{fig:robust-v-example}
\end{figure}

\section{Value Space Geometry of RMDPs}
\label{sec:main-sec}
\subsection{Policy Agreement and the Conic Hypersurface}
\label{sec:conic-hypersurface}

Recall that in Section~\ref{sec:revisit}, our new perspective connects the value space to the hyperplanes where $f_P(Y^{\pi_s})$ lies. Thus in order to characterize the robust value space, we start with studying the geometric properties of robust value functions for all policies that agree on one state, \ie, $f_\mathcal{P}(Y^{\pi_s})$. Unlike the non-robust case, $f_\mathcal{P}(Y^{\pi_s})$ may not lie in a hyperplane, as shown in Figure~\ref{fig:robust-v-example}(a). Nevertheless, it looks like $f_\mathcal{P}(Y^{\pi_s})$ still lies in a hypersurface (also see the example for $\lvert\mathcal{S}\rvert=3$ in the supplementary). In what follows, we are going to characterize this hypersurface.

First, as shown in Figure~\ref{fig:robust-v-example}(b), for different $P\in\mathcal{P}$ that share the same $P_s$, their $f_P(Y^{\pi_s})$ lie in the same hyperplane $H^{\pi_s, P_s}$. Comparing Figure~\ref{fig:robust-v-example}(a) and (b), it seems that the robust value functions $f_{\mathcal{P}}(Y^{\pi_s})$ always lie in the lower half-space $\mathcal{H}^{\pi_s, P_s}_{-}$ for different $P\in\mathcal{P}$.
On the other hand, from Eqn.~\eqref{eqn:s-rect-exist}, we know that there exists $P_s\in\mathcal{P}_s$ such that $V^{\pi,\mathcal{P}}$ lies in the hyperplane $H^{\pi_s, P_s}$.
Putting it together, we have the following lemma about $f_{\mathcal{P}}(Y^{\pi_s})$.
\begin{restatable}{lemma}{robustvboundary}
    Consider an $s$-rectangular uncertainty set $\mathcal{P}$ and a policy $\pi$, we have for all states $s\in\mathcal{S}$,
\begin{equation}
    f_\mathcal{P}(Y^{\pi_s}) \subseteq \left[ \bigcap_{P_s\in\mathcal{P}_s} \mathcal{H}^{\pi_s, P_s}_{-} \right] \cap \left[ \bigcup_{P_s\in\mathcal{P}_s} H^{\pi_s, P_s} \right].
\label{eqn:robust-halfspaces}
\end{equation}
\label{lem:robust-v-boundary}
\end{restatable}

\begin{figure}[t]
    \centering
    \includegraphics[width=\linewidth]{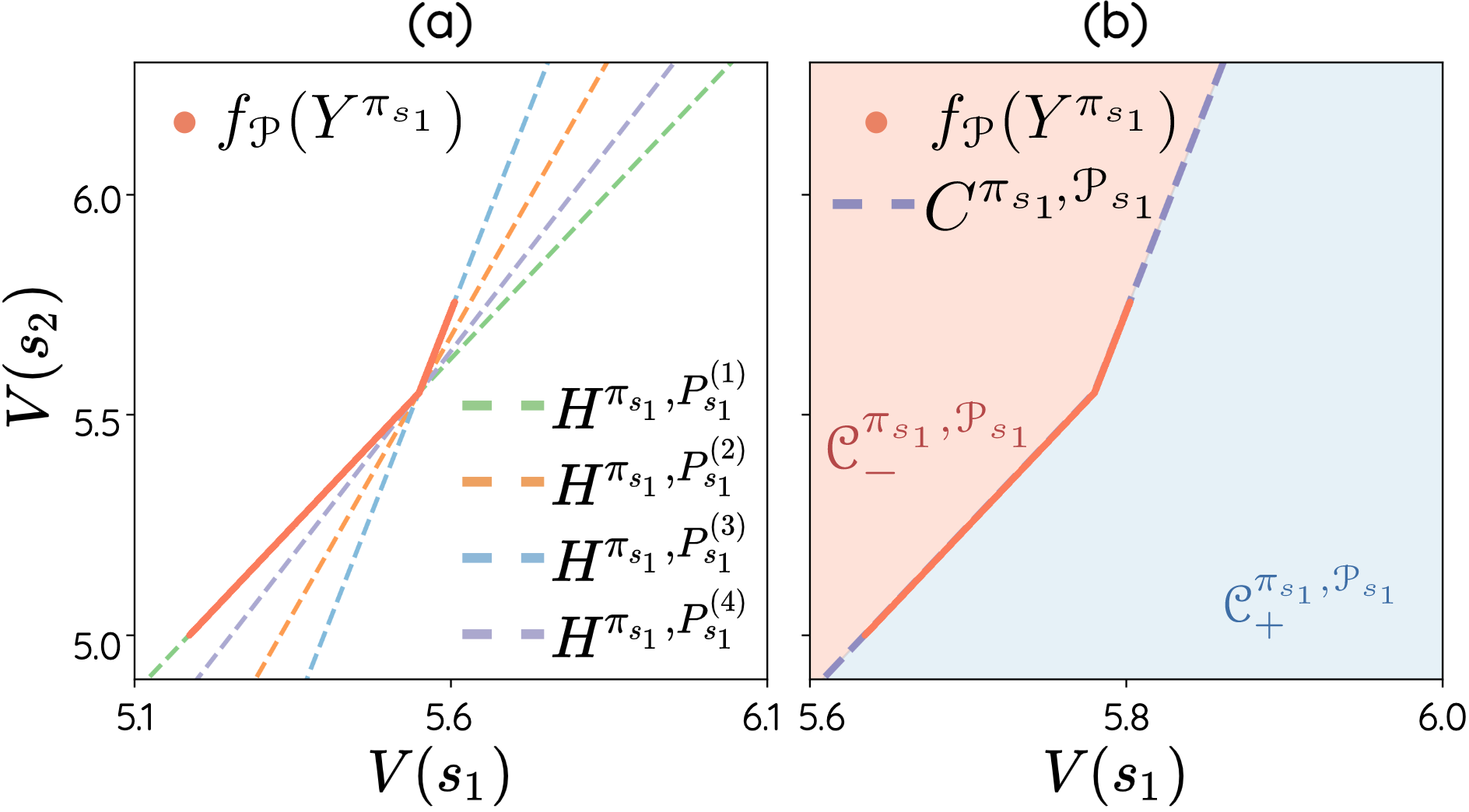}
    \caption{\textbf{(a)} For different $P_{s}\in\mathcal{P}_{s}$, the hyperplanes $H^{\pi_{s}, P_{s}}$ intersect at one point. \textbf{(b)} Illustration of the conic hypersurface in which $f_{\mathcal{P}}(Y^{\pi_s})$ lies.}
     \label{fig:conic-hypersurface}
\end{figure}

Note that the right hand side (RHS) of above Eqn.~\eqref{eqn:robust-halfspaces} is essentially the boundary of the intersection of half-spaces $\bigcap_{P_s\in\mathcal{P}_s} \mathcal{H}^{\pi_s, P_s}_{-}$. To further characterize the geometry, we need to know how these half-spaces intersect (equivalently how the hyperplanes intersect). One interesting observation is that when $\mathcal{P}_s$ contains more then 2 elements, the hyperplanes still intersect at one point, as illustrated in Figure~\ref{fig:conic-hypersurface}(a). The following lemma states this property and also gives the intersecting point.
\begin{restatable}{lemma}{intersect}
    Consider an $s$-rectangular uncertainty set $\mathcal{P}$ and a policy $\pi$, we have for all states $s\in\mathcal{S}$,
\begin{equation}
    \frac{r^{\pi_s}}{1-\gamma}\mathbf{1}\in \bigcap_{P_s\in\mathcal{P}_s} H^{\pi_s, P_s}.
\end{equation}
\label{lem:intersect}
\end{restatable}
Since the hyperplanes intersect at the same point, the intersection of the half-spaces will be a convex cone. We denote
\begin{equation}
    \begin{aligned}
    \mathcal{C}^{\pi_{s},\mathcal{P}_s}_{+} &= \{\mathbf{x}\in\mathbb{R}^\mathcal{S}\mid \langle \mathbf{x}, L^{\pi_s,P_s} \rangle \ge r^{\pi_s},\, \exists P_s\in\mathcal{P}_s \}, \\
    \mathcal{C}^{\pi_{s},\mathcal{P}_s}_{-} &= \{\mathbf{x}\in\mathbb{R}^\mathcal{S}\mid \langle \mathbf{x}, L^{\pi_s,P_s} \rangle \le r^{\pi_s},\, \forall P_s\in\mathcal{P}_s \}.
\end{aligned}
\label{eqn:conic-half-spaces}
\end{equation}
The following corollary characterizes the hypersurface that  $f_\mathcal{P}(Y^{\pi_s})$ lies in. Figure~\ref{fig:conic-hypersurface}(b) gives an illustration.
\begin{restatable}{corollary}{conicsurface}
Consider an $s$-rectangular uncertainty set $\mathcal{P}$ and a policy $\pi$, we have for all states $s\in\mathcal{S}$,
\begin{equation}
    f_\mathcal{P}(Y^{\pi_s}) \subseteq C^{\pi_s, \mathcal{P}_s}
\end{equation}
where $C^{\pi_s, \mathcal{P}_s}= \mathcal{C}^{\pi_s, \mathcal{P}_s}_{+} \cap \mathcal{C}^{\pi_s, \mathcal{P}_s}_{-}$ is a conic hypersurface.
\label{col:conic-surface}
\end{restatable}

\subsection{The Robust Value Space}

\begin{figure}[t]
    \centering
    \includegraphics[width=\linewidth]{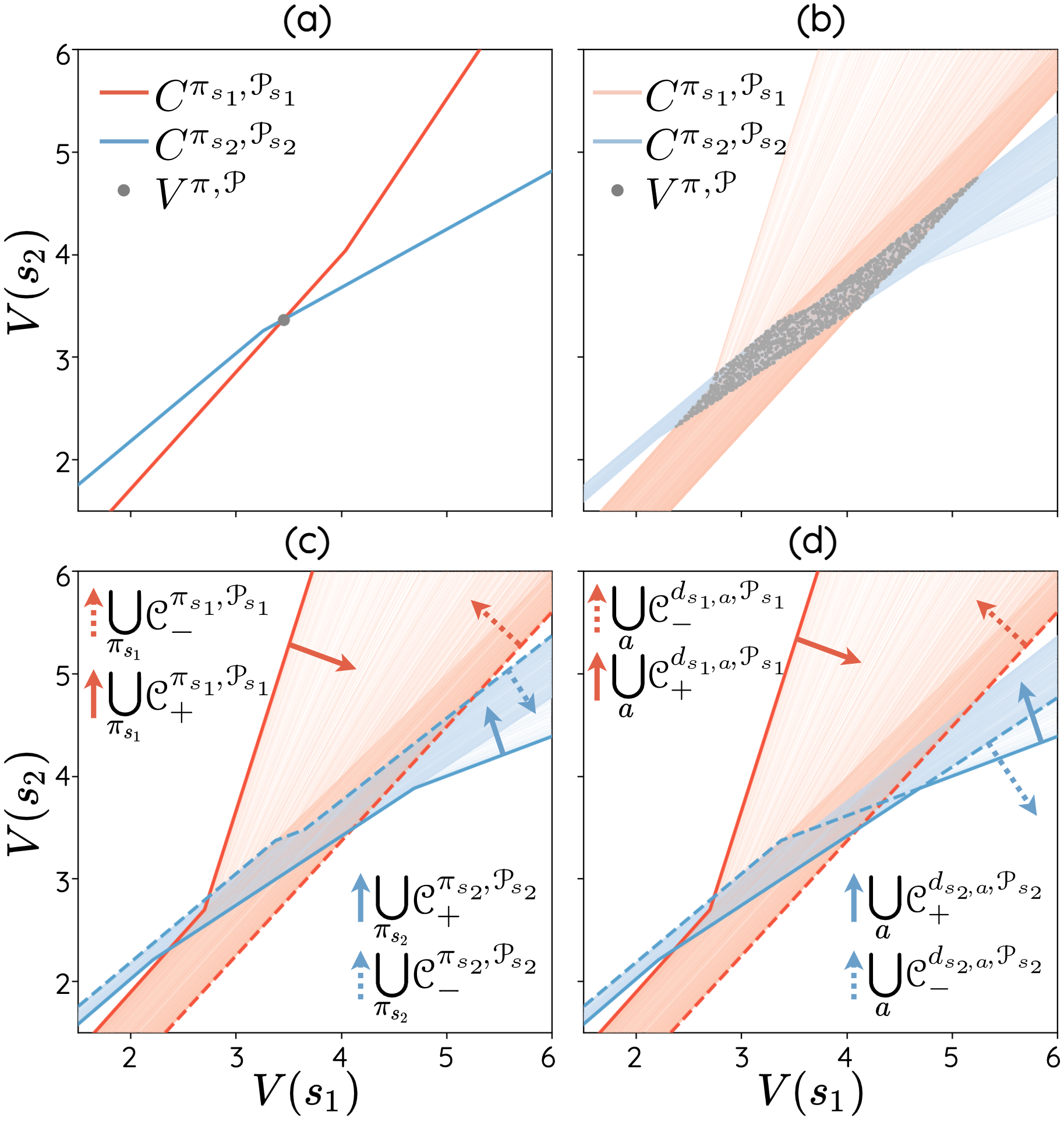}
    \caption{Visualizations of the robust value functions for a 2-state 2-action RMDP with an $s$-rectangular uncertainty set. \textbf{(a)} For a fixed $\pi$, the conic hypersurfaces $C^{\pi_s, \mathcal{P}_s}$ corresponding to different $s$ intersect at the robust value function $V^{\pi,\mathcal{P}}$. \textbf{(b)}
    For each policy $\pi$, the robust value function $V^{\pi,\mathcal{P}}$ is plotted, and the corresponding conic hypersurfaces $C^{\pi_{s}, \mathcal{P}_{s}}$ intersecting at $V^{\pi,\mathcal{P}}$ are plotted in different colors for different states.
    \textbf{(c)} For each state $s$, the union of $\mathcal{C}^{\pi_s, \mathcal{P}_s}_{+}$ and the union of $\mathcal{C}^{\pi_s, \mathcal{P}_s}_{-}$ over all $\pi_s\in\Delta_\mathcal{A}$ are highlighted respectively. \textbf{(d)} For each state $s$, the union of $\mathcal{C}^{d_{s,a}, \mathcal{P}_s}_{+}$ and the union of $\mathcal{C}^{d_{s,a}, \mathcal{P}_s}_{-}$ over all $a\in\mathcal{A}$ are highlighted respectively.}
     \label{fig:robust-vf-space}
\end{figure}

With the knowledge about the geometry of $f_\mathcal{P}(Y^{\pi_s})$, we are now ready to characterize the entire robust value space $\mathcal{V}^{\mathcal{P}}$. Similar to Section~\ref{sec:revisit}, we first connect the single robust value function to the intersection
of $S$ different conic hypersurfaces by the following lemma (see Figure~\ref{fig:robust-vf-space}(a) for an illustration).
\begin{restatable}{lemma}{robustsinglev}
    Consider an $s$-rectangular uncertainty set $\mathcal{P}$ and a policy $\pi$, we have
\begin{equation}
    \{V^{\pi, \mathcal{P}}\} = \bigcap_{s\in\mathcal{S}} C^{\pi_s, \mathcal{P}_s}.
\end{equation}
\label{lem:robust-single-v}
\end{restatable}
Then from the introduced perspective, we show that the robust value space can also be viewed as an intersection of state-wise unions of conic hypersurfaces, as illustrated in Figure~\ref{fig:robust-vf-space}(b) and formally stated in Lemma~\ref{lem:robust-vf-space}.
\begin{restatable}{lemma}{robustvfspace}
    Consider an $s$-rectangular uncertainty set $\mathcal{P}$, the robust value function space $\mathcal{V}^{\mathcal{P}}$ can be represented as
\begin{equation}
    \mathcal{V}^\mathcal{P} = \bigcup_{\pi\in\Pi} \bigcap_{s\in\mathcal{S}}\,  C^{\pi_s,\mathcal{P}_s} = \bigcap_{s\in\mathcal{S}}\, \bigcup_{\pi_s\in\Delta_\mathcal{A}} C^{\pi_s,\mathcal{P}_s}.
\end{equation}
\label{lem:robust-vf-space}
\end{restatable}
Next, we show the equivalence between each inner union in RHS of the above equation and an intersection of two unions in Lemma~\ref{lem:robust-vf-space-2}. Figure~\ref{fig:robust-vf-space}(c) gives an illustration.
Similar to the non-robust case, Lemma~\ref{lem:robust-vf-space-2} will help us characterize the relationship between the robust value space $\mathcal{V}^\mathcal{P}$ and the conic hypersurfaces corresponding to $d_{s,a}$.

\begin{restatable}{lemma}{robustvfspacetwo}
    Consider an $s$-rectangular uncertainty set $\mathcal{P}$, we have for all states $s\in\mathcal{S}$,
\begin{equation}
  \bigcup_{\pi_s\in\Delta_\mathcal{A}}\!\!C^{\pi_s,\mathcal{P}_s} = 
  \left[\bigcup_{\pi_s\in\Delta_\mathcal{A}}\!\! \mathcal{C}^{\pi_s,\mathcal{P}_s}_{+}\right] \cap \left[\bigcup_{\pi_s\in\Delta_\mathcal{A}}\!\! \mathcal{C}^{\pi_s,\mathcal{P}_s}_{-}\right]\!.
\end{equation}
\label{lem:robust-vf-space-2}
\end{restatable}
As shown in Figure~\ref{fig:robust-vf-space}(d), unlike the non-robust case, the infinite union $\bigcup_{\pi_s\in\Delta_\mathcal{A}} \mathcal{C}^{\pi_s,\mathcal{P}_s}_{-}$ does not necessarily coincides with the finite union $\bigcup_{a\in\mathcal{A}} \mathcal{C}^{d_{s,a},\mathcal{P}_s}_{-}$. The following Lemma~\ref{lem:robust-vf-polyhedron} characterizes their relationship.
\begin{restatable}{lemma}{robustvfpolyhedronlem}
Consider an $s$-rectangular uncertainty set $\mathcal{P}$, we have for all states $s\in\mathcal{S}$,
\begin{align}
    \bigcup_{\pi_s\in\Delta_\mathcal{A}} \mathcal{C}^{\pi_s,\mathcal{P}_s}_{+} & = \bigcup_{a\in\mathcal{A}} \mathcal{C}^{d_{s,a},\mathcal{P}_s}_{+}, \\ \bigcup_{\pi_s\in\Delta_\mathcal{A}} \mathcal{C}^{\pi_s,\mathcal{P}_s}_{-} & \supseteq \bigcup_{a\in\mathcal{A}} \mathcal{C}^{d_{s,a},\mathcal{P}_s}_{-}, \label{eqn:lower-bound}
\end{align}
where the equality in the second line holds when $\mathcal{P}$ is $(s,a)$-rectangular.
\label{lem:robust-vf-polyhedron}
\end{restatable}
\begin{figure}[t]
    \centering
    \includegraphics[width=\linewidth]{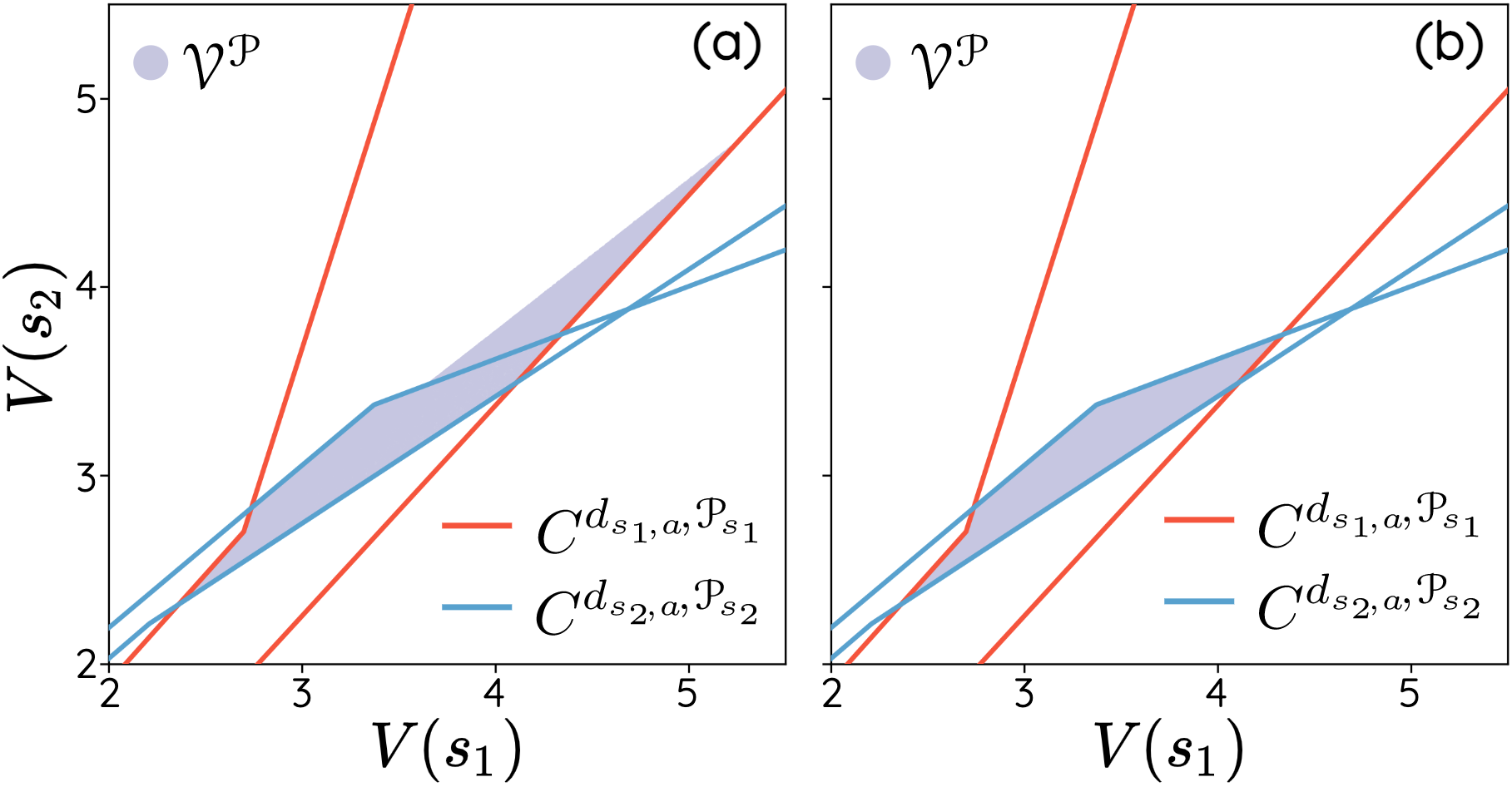}
    \caption{The robust value space $\mathcal{V}^\mathcal{P}$ and the conic hypersurfaces $C^{d_{s,a},\mathcal{P}_s}$ under $s$-rectangularity \textbf{(a)} and $(s,a)$-rectangularity \textbf{(b)}.}
     \label{fig:robust-vf-polyhedron}
\end{figure}
Putting it together, the robust value space can be characterized in Theorem~\ref{thm:robust-vf-polyhedron}.
Figure~\ref{fig:robust-vf-polyhedron} highlights the difference in the robust value space between $s$-rectangularity and $(s,a)$-rectangularity, by using the same set of probability values (see Appendix~\ref{app:specs}).
Our results also provide a geometric perspective on why the optimal policies under $s$-rectangularity might be stochastic, which is only exemplified in prior works~\cite{wiesemann2013robust}. The robust value functions of deterministic policies always lie in the region defined by RHS of Eqn.~\eqref{eqn:robust-vf-polyhedron} but the optimal value might lie outside.

\begin{restatable}{theorem}{robustvfpolyhedronthm}
Consider an $s$-rectangular uncertainty set $\mathcal{P}$, the robust value function space $\mathcal{V}^{\mathcal{P}}$ satisfies
\begin{equation}
\begin{aligned}
    \mathcal{V}^{\mathcal{P}} & = \bigcap_{s\in \mathcal{S}}\left[ \left[\bigcup_{a\in\mathcal{A}} \mathcal{C}^{d_{s,a},\mathcal{P}_s}_{+}\right] \cap \left[\bigcup_{\pi_s\in\Delta_\mathcal{A}} \mathcal{C}^{\pi_s,\mathcal{P}_s}_{-}\right] \right]
    \\ 
    & \supseteq \bigcap_{s\in \mathcal{S}}\left[ \left[\bigcup_{a\in\mathcal{A}} \mathcal{C}^{d_{s,a},\mathcal{P}_s}_{+}\right] \cap \left[\bigcup_{a\in\mathcal{A}} \mathcal{C}^{d_{s,a},\mathcal{P}_s}_{-}\right] \right]
\end{aligned}
\label{eqn:robust-vf-polyhedron}
\end{equation}
where the equality in the second line holds when $\mathcal{P}$ is $(s,a)$-rectangular.
\label{thm:robust-vf-polyhedron}
\end{restatable}

\begin{figure}[t]
    \centering
    \includegraphics[width=\linewidth]{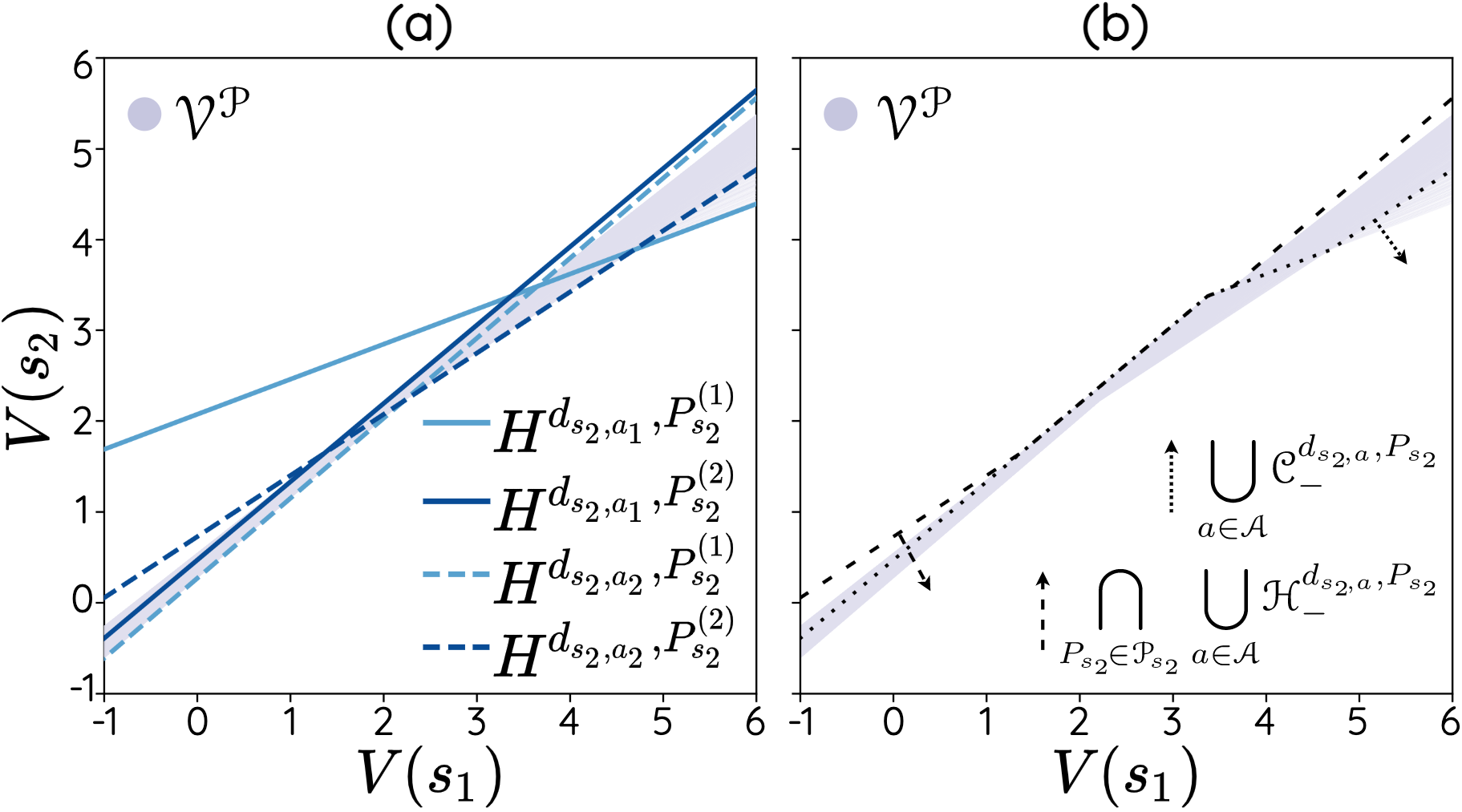}
    \caption{A closer look at the ``extra'' region under $s$-rectangularity. Here $\mathcal{P}_{s_2}=\{P_{s_2}^{(1)},P_{s_2}^{(2)}\}$. We highlight the hyperplanes in \textbf{(a)}, and the upper and lower bounds of the region    $\bigcup_{\pi_s\in\Delta_\mathcal{A}} \mathcal{C}^{\pi_s, \mathcal{P}_s}$ in \textbf{(b)}. Note that the black boundaries in \textbf{(b)} are composed by the hyperplanes in \textbf{(a)}.}
     \label{fig:extra-region}
\end{figure}

Furthermore, we take a closer look at this ``extra'' region under $s$-rectangularity. Since the space can be decomposed state-wisely, we focus on a single state $s$. Recall the definition of $\mathcal{C}_{-}^{\pi_s, \mathcal{P}_s}$ in Eqn.~\eqref{eqn:conic-half-spaces}, \ie,
\begin{equation}
    \mathcal{C}_{-}^{\pi_s, \mathcal{P}_s} = \bigcap_{P_s\in\mathcal{P}_s} \mathcal{H}_{-}^{\pi_s,P_s}.
\end{equation}
From our results in Section~\ref{sec:revisit}, we know
\begin{equation}
    \mathcal{H}_{-}^{\pi_s,P_s} \subseteq \bigcup_{a\in\mathcal{A}} \mathcal{H}_{-}^{d_{s,a},P_s}.
\end{equation}
Therefore, we can obtain
\begin{equation}
    \mathcal{C}_{-}^{\pi_s, \mathcal{P}_s} \subseteq \bigcap_{P_s\in\mathcal{P}_s} \bigcup_{a\in\mathcal{A}} \mathcal{H}_{-}^{d_{s,a},P_s},
\end{equation}
and accordingly
\begin{equation}
    \bigcup_{\pi_s\in\Delta_\mathcal{A}} \mathcal{C}_{-}^{\pi_s, \mathcal{P}_s} \subseteq \bigcap_{P_s\in\mathcal{P}_s} \bigcup_{a\in\mathcal{A}} \mathcal{H}_{-}^{d_{s,a},P_s}.
\end{equation}
The RHS of the above equation gives us an upper bound of the region $\bigcup_{\pi_s\in\Delta_\mathcal{A}} \mathcal{C}_{-}^{\pi_s, \mathcal{P}_s}$ while the RHS of Eqn.~\eqref{eqn:lower-bound} provides a lower bound. The ``extra'' region lies within the gap between them. Figure~\ref{fig:extra-region} gives an illustration using the same RMDP example as in Figure~\ref{fig:robust-vf-polyhedron}.

\subsection{Active Uncertainty Subsets}
In above sections, we have shown that the robust value space $\mathcal{V}^\mathcal{P}$ depends on $\mathcal{P}$ in the form of a set of conic hypersurfaces $C^{\pi_s,\mathcal{P}_s}$. In this section, by taking a closer look at how $\mathcal{P}_s$ and $C^{\pi_s,\mathcal{P}_s}$ are related, we will show that only a subset $\mathcal{P}^\dagger\subseteq\mathcal{P}$ is sufficient to determine the robust value space, \ie,
\begin{equation}
    \mathcal{V}^\mathcal{P} = \mathcal{V}^{\mathcal{P}^\dagger}.
\end{equation}
We term $\mathcal{P}^\dagger$ as \emph{active} uncertainty subset, analogous to active constraints, in the sense that all $P\in\mathcal{P}^\dagger$ are active in determining the shape of the robust value space $\mathcal{V}^\mathcal{P}$.

\begin{figure}[t]
    \centering
    \includegraphics[width=\linewidth]{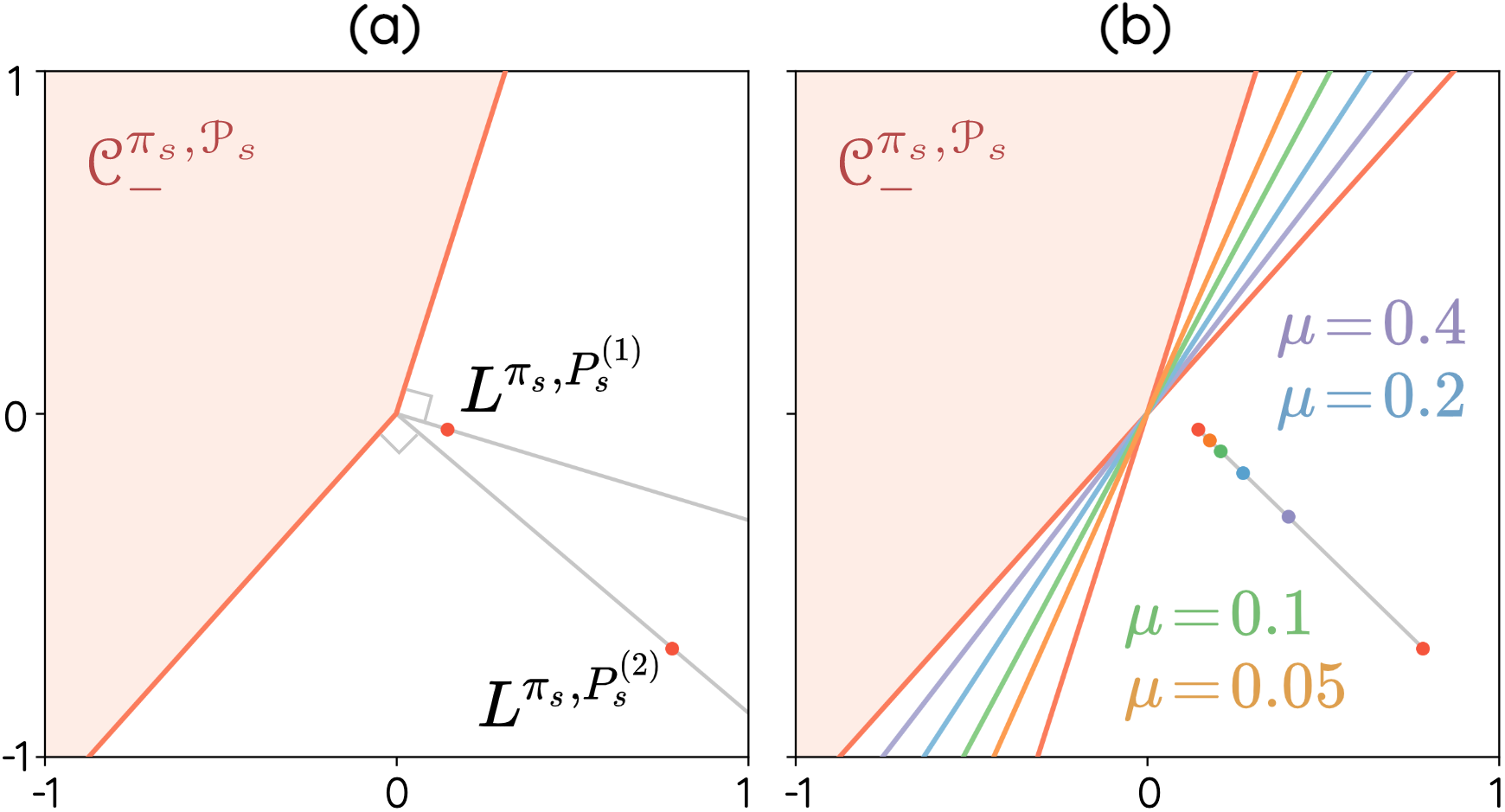}
    \caption{Visualization of the convex cone $\mathcal{C}^{\pi_s, \mathcal{P}_s}_{-}$ for a fixed $\pi_s$ and different $\mathcal{P}_s$. The translation $r^{\pi_s}\mathbf{1}$ is ignored since $\pi_s$ is fixed. \textbf{(a)} We set $\mathcal{P}_s=\{P_{s}^{(1)}, P_{s}^{(2)}\}$. \textbf{(b)} We set $\mathcal{P}_s=\{P_{s}^{(\mu)} \mid P_{s}^{(\mu)}=\mu P_{s}^{(1)} + (1-\mu) P_{s}^{(2)}, 0\le\mu\le 1\}$ and also plot the hyperplanes $H^{\pi_s, P_{s}^{(\mu)}}$ for different $\mu$.}
     \label{fig:polar-cone}
\end{figure}

First, let us keep $\pi_s$ fixed, and note that the conic hypersurface $C^{\pi_s, \mathcal{P}_s}$ is uniquely determined by the convex cone $\mathcal{C}^{\pi_s, \mathcal{P}_s}_{-}$. We then focus on the relationship between $\mathcal{P}_s$ and $\mathcal{C}^{\pi_s, \mathcal{P}_s}_{-}$. Denote the set
\begin{equation}
    \mathcal{L}^{\pi_s, \mathcal{P}_s} \coloneqq \{L^{\pi_s, P_s}\mid P_s\in\mathcal{P}_s \}.
\end{equation}
From the definition of $\mathcal{C}^{\pi_s, \mathcal{P}_s}_{-}$, we can see $\mathcal{C}^{\pi_s, \mathcal{P}_s}_{-}$ is exactly the polar cone of $\mathcal{L}^{\pi_s, \mathcal{P}_s}$ (plus a translation), denoted with
\begin{equation}
    \mathcal{C}^{\pi_s, \mathcal{P}_s}_{-} = (\mathcal{L}^{\pi_s, \mathcal{P}_s})^* + \{r^{\pi_s}\mathbf{1}\}.
\label{eqn:polar-cone}
\end{equation}
Here $+$ denotes the Minkowski addition.
Figure~\ref{fig:polar-cone}(a) gives an illustration. Note that for fixed $\pi_s$, $\mathcal{L}^{\pi_s, \mathcal{P}_s}$ is the image of $\mathcal{P}_s$ under a fixed affine transformation. We denote this affine transformation as $g$, \ie, $\mathcal{L}^{\pi_s, \mathcal{P}_s}=g(\mathcal{P}_s)$.
Then we are able to obtain the following lemma:
\begin{restatable}{lemma}{activesubsetlem}
Consider a $s$-rectangular uncertainty set $\mathcal{P}$ and a policy $\pi$, we have
\begin{equation}
    (\mathcal{L}^{\pi_s, \mathcal{P}_s})^* = (g(\mathcal{P}_s))^* = (g(\mathbf{ext}(\mathbf{conv}(\mathcal{P}_s))))^*.
\end{equation}
\label{lem:activesubsetlem}
\end{restatable}
This lemma implies that, in order to determine the conic hypersurface $C^{\pi_s, \mathcal{P}_s}$, we only need to care about those $P_s\in\mathcal{P}_s$ that are extreme points of the convex hull. Figure~\ref{fig:polar-cone}(b) gives an illustration. We then generalize it to the whole robust value space and present the following theorem:
\begin{restatable}{theorem}{activesubsetthm}
Consider a $s$-rectangular uncertainty set $\mathcal{P}$, we have
\begin{equation}
    \mathcal{V}^\mathcal{P} = \mathcal{V}^{\mathcal{P}^\dagger}
\end{equation}
where $\mathcal{P}^\dagger=\mathbf{ext}(\mathbf{conv}(\mathcal{P}))\subseteq\mathcal{P}$.
\end{restatable}
If the $\mathcal{P}$ (or more generally $\mathbf{conv}(\mathcal{P}$)) is polyhedral, such as $\ell_1$-ball and $\ell_\infty$-ball~\cite{ho2018fast, ho2021partial, behzadian2021fast}, then we can reduce $\mathcal{P}$ to a finite set without losing any useful information for policy optimization. In addition,  $\mathbf{conv}(\mathcal{P})$ being polyhedral implies that $\mathcal{C}^{\pi_s, \mathcal{P}_s}_{-}$ is a polyhedral cone. Combining with Theorem~\ref{thm:robust-vf-polyhedron}, it means that the robust value space for an $(s,a)$-rectangular uncertainty set will be a polytope.

\section{Discussion}
\label{sec:discuss}

\subsection{Policy Agreement on More States}
\label{sec:sub-polytope}

\begin{figure}[t]
    \centering
    \includegraphics[width=\linewidth]{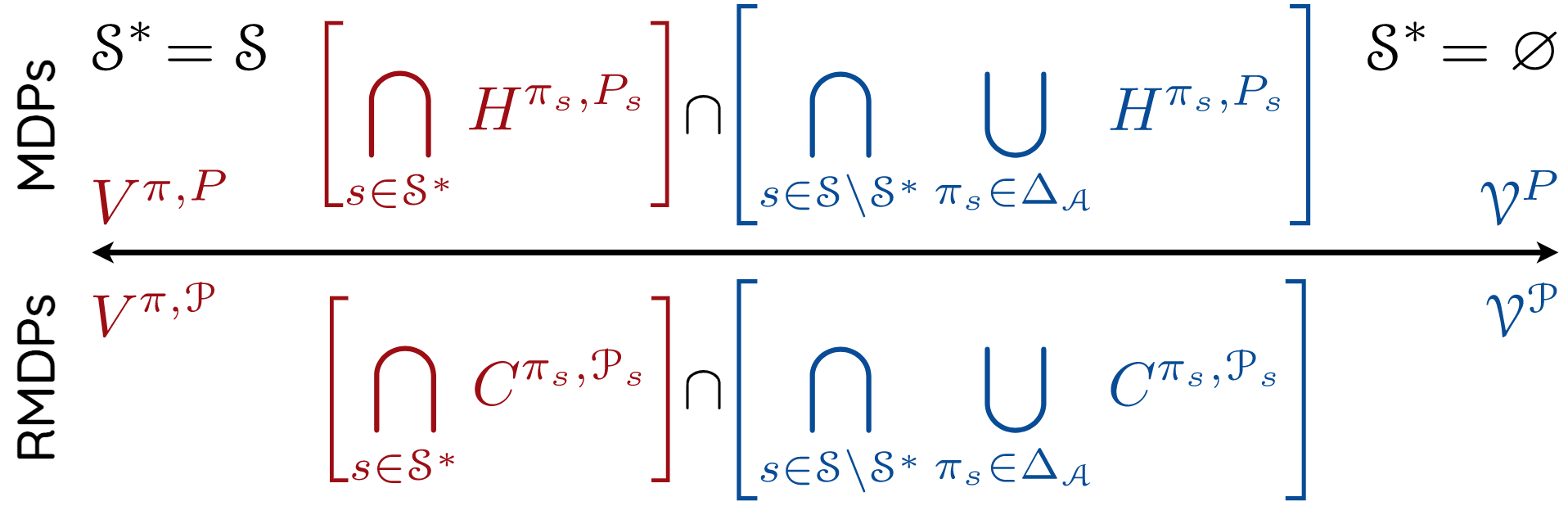}
    \caption{A spectrum of the spaces of value functions.}
     \label{fig:sub-polytope}
\end{figure}

We already know that the value functions for policies that agree on a single state lie in a hyperplane for MDPs~\cite{dadashi2019value}, and a conic hypersurface for $s$-rectangular RMDPs (Section~\ref{sec:conic-hypersurface}).
One natural question is how the space of value functions looks like when we fix the policies at more states.
With our new decomposition-based perspective, the results are immediately available from Lemma~\ref{lem:vf-space} and Lemma~\ref{lem:robust-vf-space}.

In Figure~\ref{fig:sub-polytope}, we show the space of value functions for policies agree on states in $\mathcal{S}^\ast\subseteq\mathcal{S}$, under both non-robust and robust setting.
Moreover, as illustrated in Figure~\ref{fig:sub-polytope}, our perspective reveals a spectrum of the spaces of value functions.
When the policies agree on all states, then it reduces to a single value function.
When the policies are free to vary on all states, then it is the whole value space.
This perspective enables us to characterize every point on this spectrum in an explicit form.
In comparison, for non-robust case, prior works~\cite{dadashi2019value} only prove that the spaces are polytopes without giving a clear characterization.

\subsection{The Non-convexity of the Robust Value Space}
Like the non-robust case, the robust value space $\mathcal{V}^\mathcal{P}$ is also possibly non-convex (\eg,  Figure~\ref{fig:robust-vf-polyhedron}).
Despite the non-convexity, $\mathcal{V}^\mathcal{P}$ exhibits some interesting properties analogous to monotone polygons. As shown in Figure~\ref{fig:non-convexity}(a), for any point in the robust value space $\mathcal{V}^\mathcal{P}$, if we draw an axis-parallel line passing this point, the intersection will be a line segment (or a point in degenerated case).
We formalize this observation in the following corollary.
\begin{restatable}{corollary}{robustvfsegment}
Consider an $s$-rectangular uncertainty set $\mathcal{P}$, if an axis-parallel line intersects with the robust value space $\mathcal{V}^\mathcal{P}$, then the intersection will be a line segment.
\end{restatable}

\begin{figure}[t]
    \centering
    \includegraphics[width=\linewidth]{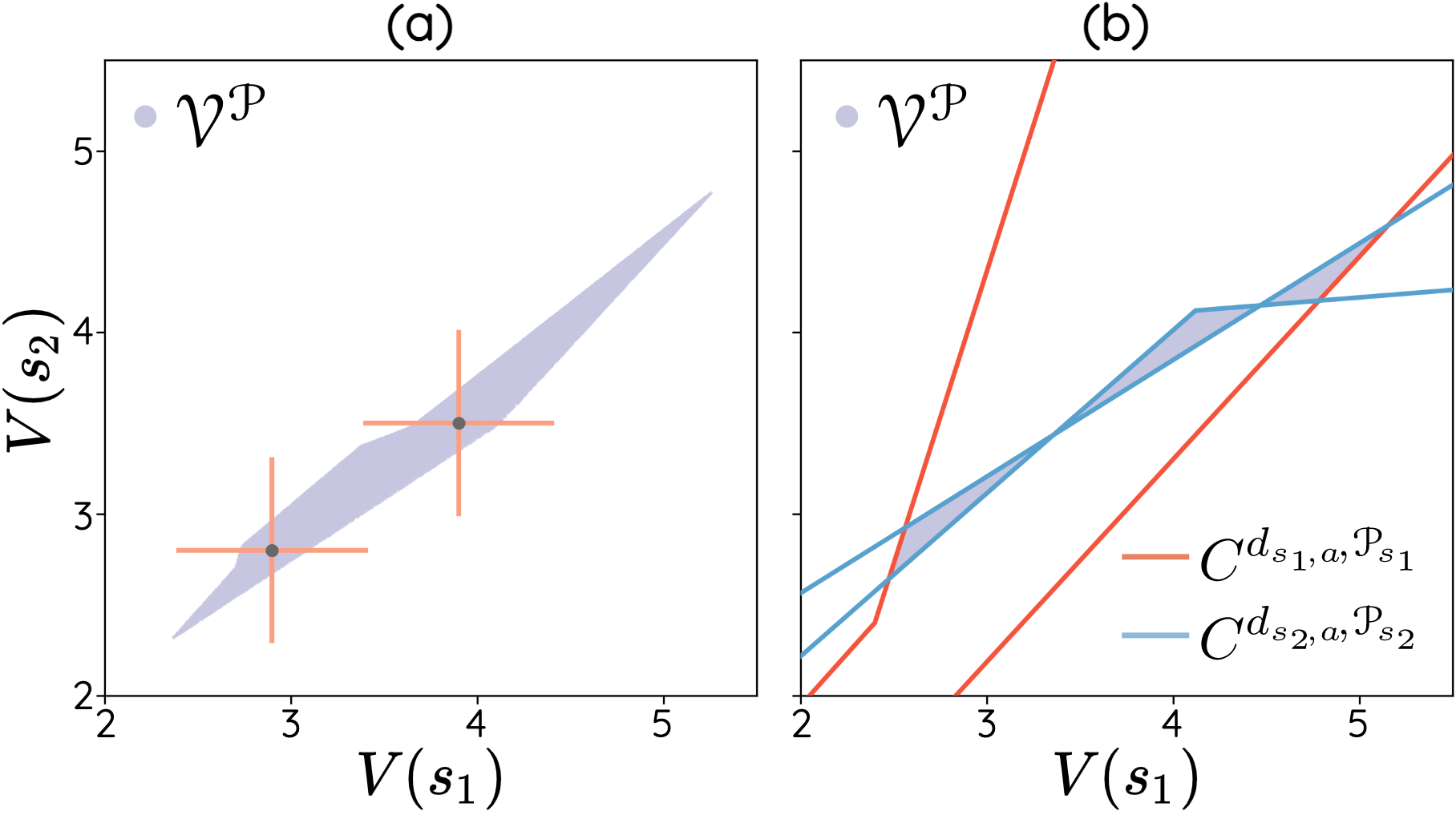}
    \caption{\textbf{(a)} The intersection between the robust value space and axis-parallel lines are line segments. \textbf{(b)} An example showing that the robust value space is not star-convex.}
     \label{fig:non-convexity}
\end{figure}
From the examples in Figure~\ref{fig:robust-vf-polyhedron}, one may wonder if the robust value function space $\mathcal{V}^\mathcal{P}$ is a star-convex set. For many randomly generated RMDPs, $\mathcal{V}^\mathcal{P}$ does look like a star-convex set (see Figure~\ref{fig:star-convexity} in Appendix~\ref{app:figure}). However, we show a carefully crafted counter-example in Figure~\ref{fig:non-convexity}(b), which is clearly not star-shaped. Nevertheless, it seems to be a rare case.
One interesting question to explore in the future is, how non-convex the robust value space can be and how likely it exhibits such non-convexity.
If it is nearly convex for most time, then we might be able to design some efficient algorithms tailored for such case.

\subsection{The Line Theorem for RMDPs}
\label{sec:line-theorem}
As mentioned before, one major obstacle that prevents us from adapting the prior method~\cite{dadashi2019value} from MDPs to RMDPs is that deriving a robust counterpart of the Line Theorem is highly challenging. Here we elaborate on this issue, with the help of our findings about the robust value space.
Without loss of generality, suppose the set of policies only differ on $s_1$.
From the discussions in Section~\ref{sec:sub-polytope}, we know the resulting set of robust value functions is
\begin{equation}
   \left[\bigcap_{i=2}^S C^{\pi_{s_i},\mathcal{P}_{s_i}} \right] \cap \left[\bigcup_{\pi_{s_1}\in\Delta_\mathcal{A}} C^{\pi_{s_1},\mathcal{P}_{s_1}} \right].
\end{equation}
The first term is an intersection of $S-1$ conic hypersurfaces and the second term is an infinite union of conic hypersurfaces. 
Both are hard to further characterize. 
For example, though we know the first term could be a curve, it is challenging to give a closed-form expression for it. 
In comparison, for MDPs, the first term is just a line and its direction is known (see the proof of Lemma 4 (ii) in ~\cite{dadashi2019value}).

\section{Related Works}
\label{sec:related}

The geometry of the space of value functions has been studied only recently. \citet{dadashi2019value} first investigate it, and establish that for MDPs the value space is a possibly non-convex polytope. Their results provide a geometric perspective to help understand the dynamics of different RL algorithms~\cite{kumar2019generalized, chan2020inverse, harb2020policy, chan2021greedification}, and also inspire new methods in representation learning in RL~\cite{bellemare2019geometric, dabney2021value}. In RMDP literature, some works take advantage of the geometric properties of special uncertainty sets to design efficient algorithms~\cite{ho2018fast, behzadian2021fast, ho2021partial}, but no prior works studies the geometry of the robust value space.

Our work can be viewed as an extension of~\cite{dadashi2019value} to RMDPs. We introduce a new perspective to characterize the geometric properties of the value space for RMDPs. Our approach also leads to a finer characterization of the value function polytope in MDPs setting.

\section{Conclusion and Future Work}
\label{sec:conclusion}

In this work, we characterize the geometry of the space of robust value functions from a new perspective, where the value space is decomposed in a state-wise manner.
We show that the robust value space is determined by a set of conic hypersurfaces. Furthermore, we can reduce the uncertainty set to a subset of extreme points without sacrificing any useful information for policy optimization.

There remain some interesting open questions. As discussed in Section~\ref{sec:discuss}, it is worth studying how non-convex the robust value space can be (\ie, can it be approximated as a convex set?). A further question is whether the level of non-convexity increases or decreases with the number of states/actions. Another direction is to investigate the geometry for other uncertainty set, such as coupled uncertainty~\cite{mannor2012lightning}, $r$-rectangular sets~\cite{goyal2018robust} or more general ones.
In addition, as in the non-robust case, it is interesting to study the geometry of robust value functions when the state space is very large and some approximation is needed.
We will leave these questions to future works.

\section*{Acknowledgements}

This work was partially supported by the Israel Science Foundation under contract 2199/20.
We appreciate the valuable feedback from ICML anonymous reviewers. We also thank Bingyi Kang and Pengqian Yu for some helpful discussions about RMDPs.


\bibliography{example_paper}
\bibliographystyle{icml2022}

\newpage
\appendix
\onecolumn

\section{Details of MDPs and RMDPs}
\label{app:specs}

In this section, we give the specifics of the MDPs and the RMDPs used for illustrations in this work.

Figure~\ref{fig:single-v}(a) and Figure~\ref{fig:vf-space}:
\begin{align*}
    & S = 2, A = 3 \\
    & r_{s_1} = (0.0199, 0.6097, 0.8313), r_{s_2} = (0.4044, 0.5534, 0.8319) \\
    & P_{s_1, a_1} = (0.7793, 0.2207), P_{s_1, a_2} = (0.9713, 0.0287), P_{s_1, a_3} = (0.0668, 0.9332) \\
    & P_{s_2, a_1} = (0.0676, 0.9324), P_{s_2, a_2} = (0.5929, 0.4071), P_{s_2, a_3} = (0.2497, 0.7503) \\
    & \pi_{s_1} = (0.2, 0.3, 0.5), \pi_{s_2} = (0.3, 0.1, 0.6)
\end{align*}

Figure~\ref{fig:single-v}(b):
\begin{align*}
    & S = 3, A = 2 \\
    & r_{s_1} = (0.5, 0.8), r_{s_2} = (0.4, 0.2), r_{s_3} = (0.2, 0.6) \\
    & P_{s_1, a_1} = (0.14, 0.75, 0.11), P_{s_1, a_2} = (0.44, 0.45, 0.11) \\
    & P_{s_2, a_1} = (0.23, 0.19, 0.58), P_{s_2, a_2} = (0.44, 0.32, 0.24) \\
    & P_{s_3, a_1} = (0.45, 0.43, 0.12), P_{s_3, a_2} = (0.14, 0.54, 0.32) \\
    & \pi_{s_1} = (0.46, 0.54), \pi_{s_2} = (0.38, 0.62),
    \pi_{s_3} = (0.49, 0.51)
\end{align*}

Figure~\ref{fig:robust-v-example}:
\begin{align*}
    & S = 2, A = 2 \\
    & r_{s_1} = (0.5, 0.6), r_{s_2} = (0.4, 0.7) \\
    & \mathcal{P}_{s_1} = \left\{
    \begin{pmatrix}
    0.78 & 0.22 \\
    0.79 & 0.21
    \end{pmatrix},
    \begin{pmatrix}
    0.85 & 0.15 \\
    0.99 & 0.01
    \end{pmatrix}
    \right\} \\
    & \mathcal{P}_{s_2} = \left\{
    \begin{pmatrix}
    0.59 & 0.41 \\
    0.92 & 0.08
    \end{pmatrix},
    \begin{pmatrix}
    0.60 & 0.40 \\
    0.39 & 0.61
    \end{pmatrix}
    \right\} \\
    & \pi_{s_1} = (0.45, 0.55), \pi_{s_2} = (0.10, 0.90)
\end{align*}

Figure~\ref{fig:conic-hypersurface}:
\begin{align*}
    & S = 2, A = 2 \\
    & r_{s_1} = (0.5, 0.6), r_{s_2} = (0.4, 0.7) \\
    & \mathcal{P}_{s_1} = \left\{
    \begin{pmatrix}
    0.78 & 0.22 \\
    0.79 & 0.21
    \end{pmatrix},
    \begin{pmatrix}
    0.85 & 0.15 \\
    0.99 & 0.01
    \end{pmatrix},
    \begin{pmatrix}
    0.92 & 0.08 \\
    0.99 & 0.01
    \end{pmatrix},
    \begin{pmatrix}
    0.92 & 0.08 \\ 
    0.83 & 0.17
    \end{pmatrix}
    \right\} \\
    & \mathcal{P}_{s_2} = \left\{
    \begin{pmatrix}
    0.59 & 0.41 \\
    0.92 & 0.08
    \end{pmatrix}
    \right\} \\
    & \pi_{s_1} = (0.45, 0.55), \pi_{s_2} = (0.10, 0.90)
\end{align*}

Figure~\ref{fig:robust-vf-space}, Figure~\ref{fig:robust-vf-polyhedron}(a), Figure~\ref{fig:extra-region}, Figure~\ref{fig:polar-cone} and Figure~\ref{fig:non-convexity}(a):
\begin{align*}
    & S = 2, A = 2 \\
    & r_{s_1} = (0.27, 0.9398), r_{s_2} = (0.3374, 0.2212) \\
    & \mathcal{P}_{s_1} = \left\{
    \begin{pmatrix}
    0.95 & 0.05 \\
    0.17 & 0.83
    \end{pmatrix},
    \begin{pmatrix}
    0.24 & 0.76 \\
    0.05 & 0.95
    \end{pmatrix}
    \right\} \\
    & \mathcal{P}_{s_2} = \left\{
    \begin{pmatrix}
    0.07 & 0.93 \\
    0.83 & 0.17
    \end{pmatrix},
    \begin{pmatrix}
    0.70 & 0.30 \\
    0.23 & 0.77
    \end{pmatrix}
    \right\} \\
    & \pi_{s_1} = (0.8, 0.2), \pi_{s_2} = (0.9, 0.1)
\end{align*}

Figure~\ref{fig:robust-vf-polyhedron}(b):
\begin{align*}
    & S = 2, A = 2 \\
    & r_{s_1} = (0.27, 0.9398), r_{s_2} = (0.3374, 0.2212) \\
    & \mathcal{P}_{s_1, a_1} = \{
    (0.95, 0.05), (0.24, 0.76)
    \} \\
    & \mathcal{P}_{s_1, a_2} = \{
    (0.17, 0.83), (0.05, 0.95)
    \} \\
    & \mathcal{P}_{s_2, a_1} = \{
    (0.07, 0.93), (0.70, 0.30)
    \} \\
    & \mathcal{P}_{s_1, a_1} = \{
    (0.83, 0.17), (0.23, 0.77)
    \}
\end{align*}

Figure~\ref{fig:non-convexity}(b):
\begin{align*}
    & S = 2, A = 2 \\
    & r_{s_1} = (0.24, 0.998), r_{s_2} = (0.3574, 0.412) \\
    & \mathcal{P}_{s_1} = \left\{
    \begin{pmatrix}
    0.95 & 0.05 \\ 
    0.05 & 0.95
    \end{pmatrix},
    \begin{pmatrix}
    0.24 & 0.76 \\
    0.95 & 0.05
    \end{pmatrix}
    \right\} \\
    & \mathcal{P}_{s_2} = \left\{
    \begin{pmatrix}
    0.2 & 0.8 \\ 
    0.99 & 0.01
    \end{pmatrix},
    \begin{pmatrix}
    0.2 & 0.8 \\ 
    0.01 & 0.99
    \end{pmatrix}
    \right\}
\end{align*}

\section{Proofs}
\label{app:proof}

\singlev*
\begin{proof} 
Observe that
\begin{equation}
    H^{\pi_s, P_s} = \{\mathbf{x}\in\mathbb{R}^\mathcal{S}\mid \langle \mathbf{x}, L^{\pi_s,P_s} \rangle = r^{\pi_s} \}
\end{equation}
is the set of vectors that satisfy the $s$-th equation of the following system of linear equations:
\begin{equation}
    (I-\gamma P^\pi)\mathbf{x} =  r^\pi.
\end{equation}
Since $(I-\gamma P^{\pi})$ is invertible, this system of linear equations has a unique solution $V^{\pi,P}$. Hence, we have
\begin{equation}
    \{V^{\pi,P}\} = \bigcap_{s\in\mathcal{S}} H^{\pi_s,P_s}
\end{equation}
which completes the proof.
\end{proof}

\vfspace*
\begin{proof}
By the definition of $\mathcal{V}^{P}$ and Lemma~\ref{lem:single-v}, we have
\begin{equation}
    \mathcal{V}^P = \bigcup_{\pi\in\Pi}\{V^{\pi,P}\} = \bigcup_{\pi\in\Pi}\bigcap_{s\in\mathcal{S}}H^{\pi_s,P_s}.
\end{equation}
We can break the union into nested unions by fixing $\pi_s$ for each $s$:
\begin{equation}
   \bigcup_{\pi\in\Pi}\bigcap_{s\in\mathcal{S}}H^{\pi_s,P_s} = \bigcup_{\pi_{s_{S}}\in\Delta_\mathcal{A}} \cdots \bigcup_{\pi_{s_{2}}\in\Delta_\mathcal{A}} \bigcup_{\pi_{s_{1}}\in\Delta_\mathcal{A}} \bigcap_{s\in\mathcal{S}}H^{\pi_s,P_s}.
\end{equation}
Then, we have
\begin{equation}
    \begin{aligned}
    \mathcal{V}^P & = \bigcup_{\pi_{s_{S}}\in\Delta_\mathcal{A}} \cdots \bigcup_{\pi_{s_{2}}\in\Delta_\mathcal{A}} \bigcup_{\pi_{s_{1}}\in\Delta_\mathcal{A}} \bigcap_{s\in\mathcal{S}}H^{\pi_s,P_s} \\
    & = \bigcup_{\pi_{s_{S}}\in\Delta_\mathcal{A}} \cdots \bigcup_{\pi_{s_{2}}\in\Delta_\mathcal{A}} \bigcup_{\pi_{s_{1}}\in\Delta_\mathcal{A}} \left[ H^{\pi_{s_1},P_{s_1}} \cap \left[ \bigcap_{i=2}^S H^{\pi_{s_i},P_{s_i}}\right]\right] \\
    & = \bigcup_{\pi_{s_{S}}\in\Delta_\mathcal{A}} \cdots \bigcup_{\pi_{s_{2}}\in\Delta_\mathcal{A}} \left[\left[\bigcup_{\pi_{s_{1}}\in\Delta_\mathcal{A}}  H^{\pi_{s_1},P_{s_1}}\right] \cap \left[ \bigcap_{i=2}^S H^{\pi_{s_i},P_{s_i}}\right]\right]. & \text{(distributive law of sets)}
\end{aligned}
\end{equation}
By iteratively applying the distributive law of sets, we can obtain
\begin{equation}
    \mathcal{V}^P = \bigcap_{s\in\mathcal{S}}\bigcup_{\pi_s\in\Delta_\mathcal{A}} H^{\pi_s,P_s}
\end{equation}
which completes the proof.
\end{proof}

\vfspacetwo*
\begin{proof}
First, by the distributive property of sets, it is trivial to obtain $\textrm{LHS}\subseteq\textrm{RHS}$. Next, we will show $\textrm{RHS}\subseteq\textrm{LHS}$. For any $\mathbf{x}\in\textrm{RHS}$, we have
\begin{equation}
    \exists\,\pi'_s, \pi''_s \in \Delta_\mathcal{A} \quad\textrm{s.t.}\quad \mathbf{x}\in \mathcal{H}^{\pi'_s,P_s}_{+} \cap \mathcal{H}^{\pi''_s,P_s}_{-}.
\end{equation}
When $\pi'_s=\pi''_s$, it is trivial to obtain $\mathbf{x}\in\textrm{LHS}$. When $\pi'_s\neq\pi''_s$, then there exists $\alpha,\beta\ge 0$ such that
\begin{equation}
    \begin{aligned}
    \langle \mathbf{x}, L^{\pi'_s,P_s} \rangle - r^{\pi'_s} & = \alpha, \\
    \langle \mathbf{x}, L^{\pi''_s,P_s} \rangle - r^{\pi''_s} & = -\beta.
\end{aligned}
\end{equation}

When either $\alpha=0$ or $\beta=0$, we have $\mathbf{x}\in H^{\pi'_s,P_s}$ or $\mathbf{x}\in H^{\pi''_s,P_s}$, and accordingly $\mathbf{x}\in\textrm{LHS}$. Therefore, we only focus on the case where $\alpha,\beta > 0$. If we set
\begin{equation}
    \pi^\dagger_s = \frac{\beta}{\alpha+\beta} \pi'_s + \frac{\alpha}{\alpha+\beta} \pi''_s,
\end{equation}
then we have
\begin{equation}
    \begin{aligned}
    \, & \langle \mathbf{x}, L^{\pi^\dagger_s,P_s} \rangle - r^{\pi^\dagger_s} \\
    =\, & \langle \mathbf{x}, \frac{\beta}{\alpha+\beta} L^{\pi'_s,P_s} + \frac{\alpha}{\alpha+\beta} L^{\pi''_s,P_s} \rangle - \frac{\beta}{\alpha+\beta} r^{\pi'_s} - \frac{\alpha}{\alpha+\beta} r^{\pi''_s} \\
    =\, & \langle \mathbf{x}, \frac{\beta}{\alpha+\beta} L^{\pi'_s,P_s} \rangle - \frac{\beta}{\alpha+\beta} r^{\pi'_s} + \langle \mathbf{x}, \frac{\alpha}{\alpha+\beta} L^{\pi''_s,P_s} \rangle - \frac{\alpha}{\alpha+\beta} r^{\pi''_s} \\
    =\, & \frac{\beta}{\alpha+\beta} \left(\langle\mathbf{x}, L^{\pi'_s,P_s} \rangle - r^{\pi'_s}\right) + \frac{\alpha}{\alpha+\beta} \left(\langle\mathbf{x}, L^{\pi''_s,P_s} \rangle - r^{\pi''_s}\right) \\
    =\, & 0.
\end{aligned}
\end{equation}

Note that $\pi^\dagger_s\in\Delta_\mathcal{A}$. The above result implies $\mathbf{x}$ lies in the hyperplane $H^{\pi^\dagger_s,P_s}$. Thus $\mathbf{x}\in\textrm{LHS}$ and accordingly $\textrm{RHS}\subseteq\textrm{LHS}$. Putting it together, we obtain $\textrm{LHS}=\textrm{RHS}$.
\end{proof}

\vfpolyhedronlem*
\begin{proof}
We first prove $\bigcup_{a\in\mathcal{A}} \mathcal{H}^{d_{s,a},P_s}_{+} = \bigcup_{\pi_s\in\Delta_\mathcal{A}} \mathcal{H}^{\pi_s,P_s}_{+}$. It is trivial that $\textrm{LHS}\subseteq\textrm{RHS}$. We then focus on proving $\textrm{RHS}\subseteq\textrm{LHS}$. For any $\mathbf{x}\in\textrm{RHS}$, we have
\begin{equation}
    \exists\,\pi'_s\in\Delta_\mathcal{A} \quad\textrm{s.t.}\quad \mathbf{x}\in\mathcal{H}^{\pi'_s,P_s}_{+}.
\end{equation}
Note that any $\pi_s\in\Delta_\mathcal{A}$ can be written as a convex combination of $d_{s,a}, a\in\mathcal{A}$. In our case, we write
\begin{equation}
    \pi'_s = \sum_{a\in\mathcal{A}} \pi'_{s,a} d_{s,a},
\end{equation}
then we have
\begin{equation}
    \begin{aligned}
    \langle \mathbf{x}, L^{\pi'_s,P_s} \rangle - r^{\pi'_s} & \ge 0 \\
    \langle \mathbf{x}, \sum_{a\in\mathcal{A}} \pi'_{s,a} L^{d_{s,a},P_s} \rangle - \sum_{a\in\mathcal{A}} \pi'_{s,a} r^{d_{s,a}} & \ge 0 \\
    \sum_{a\in\mathcal{A}} \pi'_{s,a} \left( \langle \mathbf{x}, L^{d_{s,a},P_s} \rangle - r^{d_{s,a}}\right) & \ge 0.
\end{aligned}
\end{equation}

Since $\pi'_{s,a}\ge 0$ for all $a\in\mathcal{A}$, the above inequality implies
\begin{equation}
    \exists\,a'\in\mathcal{A} \quad\textrm{s.t.}\quad \langle \mathbf{x}, L^{d_{s,a'},P_s} \rangle - r^{d_{s,a'}} \ge 0.
\end{equation}
This is equivalent to $\mathbf{x}\in\textrm{LHS}$. Putting it together, we obtain $\textrm{LHS}=\textrm{RHS}$.

The second part $\bigcup_{a\in\mathcal{A}} \mathcal{H}^{d_{s,a},P_s}_{-} = \bigcup_{\pi_s\in\Delta_\mathcal{A}} \mathcal{H}^{\pi_s,P_s}_{-}$ can be proved in the same way.
\end{proof}

\vfpolyhedron*
\begin{proof}
The first equality follow immediately from Lemma~\ref{lem:vf-space}, Lemma~\ref{lem:vf-space-2} and Lemma~\ref{lem:vf-polyhedron}. The second equality can be obtained using the distributive law of sets.
\end{proof}

\robustvboundary*
\begin{proof}
For any $\pi'\in Y^{\pi_s}$, from Eqn.~\eqref{eqn:s-rect-exist}, we know that
\begin{equation}
    \exists\,P_\dagger\in\mathcal{P}, \quad\textrm{s.t.}\quad \forall P\in\mathcal{P},\, V^{\pi', P_\dagger} \le V^{\pi', P}.
\end{equation}
Using the Bellman equation~\cite{bellman1957dynamic}, we can obtain
\begin{equation}
\begin{aligned}
   V^{\pi', P_\dagger} - V^{\pi', P} & = \gamma P^{\pi'}_\dagger V^{\pi', P_\dagger} - \gamma P^{\pi'} V^{\pi', P} \\
   & = \gamma(P^{\pi'}_\dagger - P^{\pi'}) V^{\pi',P_\dagger} - \gamma P^{\pi'} (V^{\pi', P_\dagger} - V^{\pi', P}) \\
   & = (I-\gamma P^{\pi'})^{-1} \gamma(P^{\pi'}_\dagger - P^{\pi'})V^{\pi',P_\dagger}
\end{aligned}
\end{equation}
Note that $(I-\gamma P^{\pi'})^{-1}=\sum_{t=0}^\infty (\gamma P^{\pi'})^t \ge 0$. Thus, we have
\begin{equation}
    \forall P\in\mathcal{P},\quad
    \gamma(P^{\pi'}_\dagger - P^{\pi'})V^{\pi',P_\dagger} \le 0
\end{equation}
Rearranging the above inequality, we obtain
\begin{equation}
    \forall P\in\mathcal{P},\quad
    (I - \gamma P^{\pi'})V^{\pi',P_\dagger} \le (I-\gamma P^{\pi'}_\dagger)V^{\pi',P_\dagger}.
\end{equation}
Since $(I-\gamma P^{\pi'}_\dagger)V^{\pi',P_\dagger} = r^{\pi'}$ and $V^{\pi',\mathcal{P}}=V^{\pi',P_\dagger}$, we have
\begin{equation}
    \forall P\in\mathcal{P},\quad
    (I - \gamma P^{\pi'})V^{\pi',\mathcal{P}} \le r^{\pi'}.
\end{equation}
Taking the $s$-th inequality and noting that $\pi'_s = \pi_s$, we have
\begin{equation}
    \forall P_s\in\mathcal{P}_s, \quad \langle V^{\pi', \mathcal{P}}, L^{\pi_s,P_s} \rangle \le r^{\pi_s}.
\end{equation}
Therefore, we have
\begin{equation}
    f_\mathcal{P}(Y^{\pi_s})\subseteq\bigcap_{P_s\in\mathcal{P}_s} \mathcal{H}^{\pi_s, P_s}_{-}.
\end{equation}
On the other hand, from Eqn.~\eqref{eqn:s-rect-exist} we know
\begin{equation}
    \exists\,P_s\in\mathcal{P}_s,\quad \langle V^{\pi', \mathcal{P}}, L^{\pi_s, P_s} \rangle = r^{\pi_s},
\end{equation}
which is equivalent to
\begin{equation}
    f_\mathcal{P}(Y^{\pi_s}) \subseteq \bigcup_{P_s\in\mathcal{P}_s} H^{\pi_s, P_s}.
\end{equation}
Putting it together, we get
\begin{equation}
    f_\mathcal{P}(Y^{\pi_s}) \subseteq \left[ \bigcap_{P_s\in\mathcal{P}_s} \mathcal{H}^{\pi_s, P_s}_{-}\right] \cap \left[ \bigcup_{P_s\in\mathcal{P}_s} H^{\pi_s, P_s} \right],
\end{equation}
which completes the proof.
\end{proof}

\intersect*
\begin{proof}
Recall that
\begin{equation}
    H^{\pi_s,P_s} = \{\mathbf{x}\in\mathbb{R}^\mathcal{S}\mid \langle \mathbf{x}, L^{\pi_s, P_s} \rangle = r^{\pi_s}\}.
\end{equation}
From the definition of $L^{\pi_s, P_s}$, we know
\begin{equation}
    \langle \mathbf{1}, L^{\pi_s, P_s} \rangle = \frac{1}{1-\gamma}.
\end{equation}
Thus, it is easy to verify that $\langle \frac{r^{\pi_s}}{1-\gamma}\mathbf{1}, L^{\pi_s, P_s} \rangle = r^{\pi_s}$ for all $P_s\in\mathcal{P}_s$, which concludes the proof.
\end{proof}

\conicsurface*
\begin{proof}
This corollary is a restatement of Lemma~\ref{lem:robust-v-boundary}. Note that
\begin{equation}
    \begin{aligned}
    \left[ \bigcap_{P_s\in\mathcal{P}_s} \mathcal{H}^{\pi_s, P_s}_{-}\right] \cap \left[ \bigcup_{P_s\in\mathcal{P}_s} H^{\pi_s, P_s} \right]
    = & \bigcup_{P_s\in\mathcal{P}_s}\left[\left[\bigcap_{P_s\in\mathcal{P}_s} \mathcal{H}^{\pi_s, P_s}_{-}\right] \cap  H^{\pi_s, P_s} \right] \\
    =& \bigcup_{P_s\in\mathcal{P}_s}\left[\left[\bigcap_{P_s\in\mathcal{P}_s} \mathcal{H}^{\pi_s, P_s}_{-}\right] \cap  \mathcal{H}^{\pi_s, P_s}_{+} \right] \\
    = & \left[\bigcap_{P_s\in\mathcal{P}_s} \mathcal{H}^{\pi_s, P_s}_{-}\right] \cap \left[ \bigcup_{P_s\in\mathcal{P}_s} \mathcal{H}^{\pi_s, P_s}_{+} \right] \\
    = & \,\,\mathcal{C}^{\pi_s, \mathcal{P}_s}_{-} \cap \mathcal{C}^{\pi_s, \mathcal{P}_s}_{+}.
\end{aligned}
\end{equation}

From Lemma~\ref{lem:intersect}, we know all halfspaces $\mathcal{H}^{\pi_s, P_s}_{-}$ intersect at the same point. Then their intersection $\mathcal{C}^{\pi_s, \mathcal{P}_s}_{-}$ will be a convex cone. Note that each $H^{\pi_s, P_s}$ is a supporting hyperplane of the cone $\mathcal{C}^{\pi_s, \mathcal{P}_s}_{-}$ and all $H^{\pi_s, P_s}$ determine this cone. Thus the intersection of $\bigcup_{P_s\in\mathcal{P}_s} H^{\pi_s, P_s}$ and $\mathcal{C}^{\pi_s, \mathcal{P}_s}_{-}$ is exactly the surface of $\mathcal{C}^{\pi_s, \mathcal{P}_s}_{-}$.
\end{proof}

\robustsinglev*
\begin{proof}
For any $\mathbf{x}\in\textrm{RHS}$, we have that for all $s\in\mathcal{S}$
    \begin{align}
    \exists\, P_s\in\mathcal{P}_s, \quad \langle \mathbf{x}, L^{\pi_s, P_s} \rangle = r^{\pi_s}; \\
    \forall\, P_s\in\mathcal{P}_s, \quad \langle \mathbf{x}, L^{\pi_s, P_s} \rangle \le r^{\pi_s}.
\end{align}
Since $P$ is $s$-rectangular, we have
\begin{align}
   \exists\, P\in\mathcal{P}, \quad (I-\gamma P^{\pi}) \mathbf{x} = r^{\pi}; \\  \forall\, P\in\mathcal{P}, \quad (I-\gamma P^{\pi}) \mathbf{x} \le r^{\pi}.
   \label{eqn:proof-lem-4-4}
\end{align}
Since the Bellman equation has a unique solution, the first line implies $\exists\, P_\dagger\in\mathcal{P}, \mathbf{x}=V^{\pi, P_\dagger}$. Suppose $V^{\pi, P_\dagger} \neq V^{\pi, \mathcal{P}}$, then from Eqn.~\eqref{eqn:s-rect-exist} we have
\begin{equation}
    \exists\, P_\ddagger\in\mathcal{P}, \quad\textrm{s.t.}\quad V^{\pi, P_\ddagger} = V^{\pi, \mathcal{P}} < V^{\pi, P_\dagger}.
\end{equation}
On the other hand, from Eqn.~\eqref{eqn:proof-lem-4-4}, we know
\begin{equation}
\begin{aligned}
    (I-\gamma P^{\pi}_\ddagger) V^{\pi, P_\dagger} - r^{\pi} & \le 0  \\
    (I-\gamma P^{\pi}_\ddagger) V^{\pi, P_\dagger} - (I-\gamma P^{\pi}_\dagger) V^{\pi, P_\dagger} & \le 0 \\
    \gamma (P^{\pi}_\dagger - P^{\pi}_\ddagger)V^{\pi, P_\dagger} & \le 0 \\
    (I-\gamma P^{\pi}_\ddagger)^{-1} \gamma (P^{\pi}_\dagger - P^{\pi}_\ddagger)V^{\pi, P_\dagger} & \le 0
    & \text{(see the proof of Lemma~\ref{lem:robust-v-boundary})}
    \\
    V^{\pi, P_\dagger} - V^{\pi, P_\ddagger} & \le 0 \\
    V^{\pi, P_\dagger} & \le V^{\pi, P_\ddagger}.
\end{aligned}
\end{equation}
We have an contradiction.  Therefore, we can conclude $\mathbf{x}= V^{\pi, \mathcal{P}}$ and accordingly $\{V^{\pi, \mathcal{P}}\} = \bigcap_{s\in\mathcal{S}} C^{\pi_s, \mathcal{P}_s}$.
\end{proof}

\robustvfspace*
\begin{proof}
The proof below follows exactly the same procedure as the proof of Lemma~\ref{lem:vf-space}.
By the definition of $\mathcal{V}^\mathcal{P}$ and Lemma~\ref{lem:robust-single-v}, we have
\begin{equation}
    \mathcal{V}^\mathcal{P} = \bigcup_{\pi\in\Pi}\{V^{\pi,\mathcal{P}}\} = \bigcup_{\pi\in\Pi}\bigcap_{s\in\mathcal{S}}C^{\pi_s,\mathcal{P}_s}.
\end{equation}
We can break the union into nested unions by fixing $\pi_s$ for each $s$:
\begin{equation}
   \bigcup_{\pi\in\Pi}\bigcap_{s\in\mathcal{S}}C^{\pi_s,\mathcal{P}_s} = \bigcup_{\pi_{s_{S}}\in\Delta_\mathcal{A}} \cdots \bigcup_{\pi_{s_{2}}\in\Delta_\mathcal{A}} \bigcup_{\pi_{s_{1}}\in\Delta_\mathcal{A}} \bigcap_{s\in\mathcal{S}} C^{\pi_s,\mathcal{P}_s}.
\end{equation}
Then, we have
\begin{equation}
    \begin{aligned}
    \mathcal{V}^\mathcal{P} & = \bigcup_{\pi_{s_{S}}\in\Delta_\mathcal{A}} \cdots \bigcup_{\pi_{s_{2}}\in\Delta_\mathcal{A}} \bigcup_{\pi_{s_{1}}\in\Delta_\mathcal{A}} \bigcap_{s\in\mathcal{S}}C^{\pi_s,\mathcal{P}_s} \\
    & = \bigcup_{\pi_{s_{S}}\in\Delta_\mathcal{A}} \cdots \bigcup_{\pi_{s_{2}}\in\Delta_\mathcal{A}} \bigcup_{\pi_{s_{1}}\in\Delta_\mathcal{A}} \left[ C^{\pi_{s_1},\mathcal{P}_{s_1}} \cap \left[ \bigcap_{i=2}^S C^{\pi_{s_i},\mathcal{P}_{s_i}}\right]\right] \\
    & = \bigcup_{\pi_{s_{S}}\in\Delta_\mathcal{A}} \cdots \bigcup_{\pi_{s_{2}}\in\Delta_\mathcal{A}} \left[\left[\bigcup_{\pi_{s_{1}}\in\Delta_\mathcal{A}} C^{\pi_{s_1},\mathcal{P}_{s_1}}\right] \cap \left[ \bigcap_{i=2}^S C^{\pi_{s_i},\mathcal{P}_{s_i}}\right]\right]. & \text{(distributive law of sets)}
\end{aligned}
\end{equation}

By iteratively applying the distributive law of sets, we can obtain
\begin{equation}
    \mathcal{V}^\mathcal{P} = \bigcap_{s\in\mathcal{S}}\bigcup_{\pi_s\in\Delta_\mathcal{A}} C^{\pi_s,\mathcal{P}_s},
\end{equation}
which completes the proof.
\end{proof}

\robustvfspacetwo*
\begin{proof}
Recall that $C^{\pi_s,\mathcal{P}_s}=\mathcal{C}^{\pi_s,\mathcal{P}_s}_{+} \cap \mathcal{C}^{\pi_s,\mathcal{P}_s}_{-}$, then we need to prove
\begin{equation}
    \bigcup_{\pi_s\in\Delta_\mathcal{A}} \mathcal{C}^{\pi_s,\mathcal{P}_s}_{+} \cap \mathcal{C}^{\pi_s,\mathcal{P}_s}_{-}
    = \left[\bigcup_{\pi_s\in\Delta_\mathcal{A}} \mathcal{C}^{\pi_s,\mathcal{P}_s}_{+}\right] \cap \left[\bigcup_{\pi_s\in\Delta_\mathcal{A}} \mathcal{C}^{\pi_s,\mathcal{P}_s}_{-}\right].
\end{equation}
First, by the distributive property of sets, it is trivial to obtain $\textrm{LHS}\subseteq\textrm{RHS}$. Next, we will show $\textrm{RHS}\subseteq\textrm{LHS}$. For any $\mathbf{x}\in\textrm{RHS}$, we have
\begin{equation}
    \exists\,\pi'_s, \pi''_s \in \Delta_\mathcal{A} \quad\textrm{s.t.}\quad \mathbf{x}\in \mathcal{C}^{\pi'_s,\mathcal{P}_s}_{+} \cap \mathcal{C}^{\pi''_s,\mathcal{P}_s}_{-}.
\end{equation}
When $\pi'_s=\pi''_s$, it is trivial to obtain $\mathbf{x}\in\textrm{LHS}$. When $\pi'_s\neq\pi''_s$, then we have
\begin{equation}
    \begin{aligned}
    & \exists\,P_s\in\mathcal{P}_s,\quad \langle \mathbf{x}, L^{\pi'_s,P_s} \rangle - r^{\pi'_s} \ge 0; \\
    & \forall\,P_s\in\mathcal{P}_s,\quad \langle \mathbf{x}, L^{\pi''_s,P_s} \rangle - r^{\pi''_s} \le 0.
\end{aligned}
\end{equation}

If there exists $P_s\in\mathcal{P}_s$ such that $\langle \mathbf{x}, L^{\pi''_s,P_s} \rangle - r^{\pi''_s} = 0$, then we will get $\mathbf{x}\in H^{\pi''_s,P_s}\subseteq\mathcal{C}^{\pi''_s,\mathcal{P}_s}_{+}$ and accordingly $\mathbf{x}\in\textrm{LHS}$. Therefore, we only consider the case where
\begin{equation}
    \begin{aligned}
    & \exists\,P_s\in\mathcal{P}_s,\quad \langle \mathbf{x}, L^{\pi'_s,P_s} \rangle - r^{\pi'_s} \ge 0; \\
    & \forall\,P_s\in\mathcal{P}_s,\quad \langle \mathbf{x}, L^{\pi''_s,P_s} \rangle - r^{\pi''_s} < 0.
\end{aligned}
\end{equation}

We denote
\begin{equation}
    \begin{aligned}
    \alpha^{P_s} & \coloneqq \langle \mathbf{x}, L^{\pi'_s,P_s} \rangle - r^{\pi'_s}, \\
    \beta^{P_s} & \coloneqq r^{\pi''_s} - \langle \mathbf{x}, L^{\pi''_s,P_s} \rangle, \\
    \mathscr{P}_s & \coloneqq \{P_s\mid \alpha^{P_s} \ge 0,\, \beta^{P_s} > 0,\, P_s\in\mathcal{P}_s\}, \\
    \lambda & \coloneqq \min_{P_s\in\mathscr{P}_s} \frac{\alpha^{P_s}}{\alpha^{P_s} + \beta^{P_s}}.
\end{aligned}
\end{equation}

and accordingly
\begin{equation}
    1-\lambda = \max_{P_s\in\mathscr{P}_s} \frac{\beta^{P_s}}{\alpha^{P_s} + \beta^{P_s}}.
\end{equation}
We construct
\begin{equation}
    \pi^\dagger_s \coloneqq (1-\lambda) \pi'_s + \lambda \pi''_s.
\end{equation}
Note that $0\le\lambda\le 1$. We have $\pi^\dagger_s\in\Delta_\mathcal{A}$ since $\pi^\dagger_s$ is a convex combination of $\pi'_s$ and $\pi''_s$. Then we are going to show that $\mathbf{x}\in\mathcal{C}^{\pi^\dagger_s,\mathcal{P}_s}_{+} \cap \mathcal{C}^{\pi^\dagger_s,\mathcal{P}_s}_{-}$, \ie,
\begin{equation}
    \begin{aligned}
    & \exists\,P_s\in\mathcal{P}_s,\quad \langle \mathbf{x}, L^{\pi^\dagger_s,P_s} \rangle - r^{\pi^\dagger_s} \ge 0; \\
    & \forall\,P_s\in\mathcal{P}_s,\quad \langle \mathbf{x}, L^{\pi^\dagger_s,P_s} \rangle - r^{\pi^\dagger_s} \le 0.
\end{aligned}
\end{equation}

On the one hand, denoting
\begin{equation}
    P^\dagger_s \coloneqq \argmin_{P_s\in\mathscr{P}_s} \frac{\alpha^{P_s}}{\alpha^{P_s} + \beta^{P_s}},
\end{equation}
we have
\begin{equation}
    \begin{aligned}
    \langle \mathbf{x}, L^{\pi^\dagger_s,P^\dagger_s} \rangle - r^{\pi^\dagger_s} & = \langle \mathbf{x}, (1-\lambda) L^{\pi'_s,P^\dagger_s} + \lambda L^{\pi''_s,P^\dagger_s} \rangle - (1-\lambda) r^{\pi'_s} - \lambda r^{\pi''_s} \\
    & = (1-\lambda) \left(\langle \mathbf{x}, L^{\pi'_s,P^\dagger_s} \rangle - r^{\pi'_s} \right) - \lambda \left(r^{\pi''_s} - \langle \mathbf{x}, L^{\pi''_s,P^\dagger_s} \rangle\right) \\
    & = (1-\lambda) \alpha^{P^\dagger_s} - \lambda \beta^{P^\dagger_s} \\
    & = \frac{\beta^{P^\dagger_s}\alpha^{P^\dagger_s}}{\alpha^{P^\dagger_s} + \beta^{P^\dagger_s}} - \frac{\alpha^{P\dagger_s}\beta^{P\dagger_s}}{\alpha^{P\dagger_s} + \beta^{P\dagger_s}} \\
    & = 0.
\end{aligned}
\end{equation}

On the other hand, for all $P_s\in\mathcal{P}_s$ we have \begin{equation}
    \begin{aligned}
    \langle \mathbf{x}, L^{\pi^\dagger_s,P_s} \rangle - r^{\pi^\dagger_s} & = \langle \mathbf{x}, (1-\lambda) L^{\pi'_s,P_s} + \lambda L^{\pi''_s,P_s} \rangle - (1-\lambda) r^{\pi'_s} - \lambda r^{\pi''_s} \\
    & = (1-\lambda) \left(\langle \mathbf{x}, L^{\pi'_s,P_s} \rangle - r^{\pi'_s} \right) - \lambda \left(r^{\pi''_s} - \langle \mathbf{x}, L^{\pi''_s,P_s} \rangle\right) \\
    & = (1-\lambda) \alpha^{P_s} - \lambda \beta^{P_s} \\
    & \le \frac{\beta^{P_s}\alpha^{P_s}}{\alpha^{P_s} + \beta^{P_s}} - \lambda \beta^{P_s} \\
    & \le \frac{\beta^{P_s}\alpha^{P_s}}{\alpha^{P_s} + \beta^{P_s}} - \frac{\alpha^{P_s}\beta^{P_s}}{\alpha^{P_s} + \beta^{P_s}} \\
    & \le 0.
\end{aligned}
\end{equation}
\begin{equation}
    \begin{aligned}
    \langle \mathbf{x}, L^{\pi^\dagger_s,P_s} \rangle - r^{\pi^\dagger_s} & = \langle \mathbf{x}, (1-\lambda) L^{\pi'_s,P_s} + \lambda L^{\pi''_s,P_s} \rangle - (1-\lambda) r^{\pi'_s} - \lambda r^{\pi''_s} \\
    & = (1-\lambda) \left(\langle \mathbf{x}, L^{\pi'_s,P_s} \rangle - r^{\pi'_s} \right) - \lambda \left(r^{\pi''_s} - \langle \mathbf{x}, L^{\pi''_s,P_s} \rangle\right) \\
    & = (1-\lambda) \alpha^{P_s} - \lambda \beta^{P_s} \\
    & \le \frac{\beta^{P_s}\alpha^{P_s}}{\alpha^{P_s} + \beta^{P_s}} - \lambda \beta^{P_s} \\
    & \le \frac{\beta^{P_s}\alpha^{P_s}}{\alpha^{P_s} + \beta^{P_s}} - \frac{\alpha^{P_s}\beta^{P_s}}{\alpha^{P_s} + \beta^{P_s}} \\
    & \le 0.
\end{aligned}
\end{equation}

Putting it together, we obtain $\mathbf{x}\in\mathcal{C}^{\pi^\dagger_s,\mathcal{P}_s}_{+} \cap \mathcal{C}^{\pi^\dagger_s,\mathcal{P}_s}_{-}$ and thus $\mathbf{x}\in\textrm{LHS}$.
\end{proof}

\robustvfpolyhedronlem*
\begin{proof}
\textbf{First}, we are going to prove
\begin{equation}
    \bigcup_{\pi_s\in\Delta_\mathcal{A}} \mathcal{C}^{\pi_s,\mathcal{P}_s}_{+} =
    \bigcup_{a\in\mathcal{A}} \mathcal{C}^{d_{s,a},\mathcal{P}_s}_{+}.
\end{equation}
It is trivial that $\textrm{RHS}\subseteq\textrm{LHS}$. We then focus on proving $\textrm{LHS}\subseteq\textrm{RHS}$. For any $\mathbf{x}\in\textrm{LHS}$, we have
\begin{equation}
    \exists\,\pi'_s\in\Delta_\mathcal{A} \quad\textrm{s.t.}\quad \mathbf{x}\in\mathcal{C}^{\pi'_s,\mathcal{P}_s}_{+}.
\end{equation}
Note that $\pi_s\in\Delta_\mathcal{A}$ can be written as a convex combination of $d_{s,a}, a\in\mathcal{A}$. In our case, we write
\begin{equation}
    \pi'_s = \sum_{a\in\mathcal{A}} \pi'_{s,a} d_{s,a}.
\end{equation}
Also note that for any $P_s\in\mathcal{P}_s$,
\begin{equation}
    \langle \mathbf{x}, L^{\pi'_s,P_s} \rangle - r^{\pi'_s} = \langle \mathbf{x}, \sum_{a\in\mathcal{A}} \pi'_{s,a} L^{d_{s,a},P_s} \rangle - \sum_{a\in\mathcal{A}} \pi'_{s,a} r^{d_{s,a}} = \sum_{a\in\mathcal{A}} \pi'_{s,a} \left( \langle \mathbf{x}, L^{d_{s,a},P_s} \rangle - r^{d_{s,a}}\right).
\end{equation}
Therefore, we can write $\mathbf{x}\in\mathcal{C}^{\pi'_s,\mathcal{P}_s}_{+}$ as
\begin{equation}
    \exists\,P_s\in\mathcal{P}_s,\quad \sum_{a\in\mathcal{A}} \pi'_{s,a} \left( \langle \mathbf{x}, L^{d_{s,a},P_s} \rangle - r^{d_{s,a}}\right) \ge 0.
\end{equation}
Since $\pi'_{s,a}\ge 0$ for all $a\in\mathcal{A}$, the above statement implies
\begin{equation}
    \exists\,P_s\in\mathcal{P}_s,\,\exists\,a'\in\mathcal{A} \quad\textrm{s.t.}\quad \langle \mathbf{x}, L^{d_{s,a'},P_s} \rangle - r^{d_{s,a'}} \ge 0.
\end{equation}
This is equivalent to $\mathbf{x}\in\textrm{RHS}$. Putting it together, we obtain $\textrm{LHS}=\textrm{RHS}$.

\textbf{Second}, we are going to prove
\begin{equation}
    \bigcup_{\pi_s\in\Delta_\mathcal{A}} \mathcal{C}^{\pi_s,\mathcal{P}_s}_{-}  \supseteq \bigcup_{a\in\mathcal{A}} \mathcal{C}^{d_{s,a},\mathcal{P}_s}_{-},
\end{equation}
where the equality holds when $\mathcal{P}$ is $(s,a)$-rectangular. Again, it is trivial that $\textrm{RHS}\subseteq\textrm{LHS}$. We then focus on proving $\textrm{LHS}\subseteq\textrm{RHS}$ when $\mathcal{P}$ is $(s,a)$-rectangular. For any $\mathbf{x}\in\textrm{LHS}$, we have
\begin{equation}
    \exists\,\pi'_s\in\Delta_\mathcal{A} \quad\textrm{s.t.}\quad \mathbf{x}\in\mathcal{C}^{\pi'_s,\mathcal{P}_s}_{-}.
\end{equation}
Similarly, we can obtain
\begin{equation}
    \forall\,P_s\in\mathcal{P}_s,\quad \sum_{a\in\mathcal{A}} \pi'_{s,a} \left( \langle \mathbf{x}, L^{d_{s,a},P_s} \rangle - r^{d_{s,a}}\right) \le 0.
\end{equation}
This is equivalent to
\begin{equation}
    \max_{P_s\in\mathcal{P}_s}\sum_{a\in\mathcal{A}} \pi'_{s,a} \left( \langle \mathbf{x}, L^{d_{s,a},P_s} \rangle - r^{d_{s,a}}\right) \le 0.
\end{equation}
Due to $(s,a)$-rectangularity of $\mathcal{P}$, we have
\begin{equation}
    \sum_{a\in\mathcal{A}} \pi'_{s,a} \max_{P_{s,a}\in\mathcal{P}_{s,a}} \left( \langle \mathbf{x}, L^{d_{s,a},P_s} \rangle - r^{d_{s,a}}\right) \le 0.
\end{equation}
Since $\pi'_{s,a}\ge 0$ for all $a\in\mathcal{A}$, the above statement implies
\begin{equation}
    \exists\,a'\in\mathcal{A}, \quad\textrm{s.t.}\quad \max_{P_{s,a'}\in\mathcal{P}_{s,a'}}\langle \mathbf{x}, L^{d_{s,a'},P_s} \rangle - r^{d_{s,a'}} \le 0,
\end{equation}
which is equivalent to
\begin{equation}
    \exists\,a'\in\mathcal{A},\,\forall\,P_{s,a'}\in\mathcal{P}_{s,a'} \quad\textrm{s.t.}\quad \langle \mathbf{x}, L^{d_{s,a'},P_s} \rangle - r^{d_{s,a'}} \le 0.
\end{equation}
This is essentially saying $\mathbf{x}\in\textrm{RHS}$. Putting it together, we obtain $\textrm{LHS}=\textrm{RHS}$ when $\mathcal{P}$ is $(s,a)$-rectangular.
\end{proof}

\robustvfpolyhedronthm*
\begin{proof}
The proof follows immediately from Lemma~\ref{lem:robust-vf-space}, Lemma~\ref{lem:robust-vf-space-2} and Lemma~\ref{lem:robust-vf-polyhedron}.
\end{proof}

\activesubsetlem*
\begin{proof}
Since affine transformations preserve affine hulls~\cite{dattorro2005convex}, we have
\begin{equation}
    \begin{aligned}
\mathbf{conv}(g(\mathcal{P}_s)) &= g(\mathbf{conv}(\mathcal{P}_s)), \\
\mathbf{conv}(g(\mathbf{ext}(\mathbf{conv}(\mathcal{P}_s)))) &= g(\mathbf{conv}(\mathbf{ext}(\mathbf{conv}(\mathcal{P}_s)))).
\end{aligned}
\end{equation}

Using Krein-Milman Theorem~\cite{Krein1940}, we can obtain
\begin{equation}
    g(\mathbf{conv}(\mathbf{ext}(\mathbf{conv}(\mathcal{P}_s)))) = g(\mathbf{conv}(\mathcal{P}_s)).
\end{equation}
Putting it together, we have
\begin{equation}
    \mathbf{conv}(g(\mathcal{P}_s)) = \mathbf{conv}(g(\mathbf{ext}(\mathbf{conv}(\mathcal{P}_s)))).
\end{equation}
Then by the properties of polar cones (Proposition 2.2.1
in~\cite{bertsekas2009convex}), we can get
\begin{equation}
    (g(\mathcal{P}_s))^* = (g(\mathbf{ext}(\mathbf{conv}(\mathcal{P}_s))))^*,
\end{equation}
which completes the proof.
\end{proof}

\activesubsetthm*
\begin{proof}
From Eqn.~\eqref{eqn:polar-cone} and Lemma~\ref{lem:activesubsetlem}, we know that each conic hypersurface $C^{\pi_s,\mathcal{P}_s}$ only depends on $\mathbf{ext}(\mathbf{conv}(\mathcal{P}_s))$. Then we have
\begin{equation}
    \mathcal{V}^{\mathcal{P}} = \mathcal{V}^{\mathcal{P}^\dagger}, \quad \textrm{where} \quad \mathcal{P}^\dagger= \bigtimes_{s\in\mathcal{S}} \mathbf{ext}(\mathbf{conv}(\mathcal{P}_s)).
\end{equation}
By the definition of extreme points, it is straightforward to show that
\begin{equation} \bigtimes_{s\in\mathcal{S}} \mathbf{ext}(\mathbf{conv}(\mathcal{P}_s)) = \mathbf{ext}\left(\bigtimes_{s\in\mathcal{S}} \mathbf{conv}(\mathcal{P}_s)\right).
\end{equation}
Using the properties of Cartesian products~\cite{bertsekas2003convex}, we can get
\begin{equation}
\bigtimes_{s\in\mathcal{S}} \mathbf{conv}(\mathcal{P}_s) = \mathbf{conv}\left(\bigtimes_{s\in\mathcal{S}}\mathcal{P}_s \right) = \mathbf{conv}(\mathcal{P}).
\end{equation}
Putting it together, we have $\mathcal{P}^\dagger=\mathbf{ext}(\mathbf{conv}(\mathcal{P}))$. Since $\mathcal{P}$ is assumed to be compact, then $\mathcal{P}^\dagger\subseteq\mathcal{P}$.
\end{proof}

\robustvfsegment*
\begin{proof}
Without loss of generality, consider a line parallel to the axis corresponding to state $s_1$, and denote it as
\begin{equation}
    K = \left\{\mathbf{u} + t\mathbf{e}_{s_1} \mid t\in\mathbb{R} \right\}
\end{equation}
where $\mathbf{u}\in\mathbb{R}^\mathcal{S}$ is fixed. Then the intersection between this line and the robust value space is
\begin{equation}
  K \cap \left[ \bigcap_{s\in\mathcal{S}}\left[
  \left[\bigcup_{\pi_s\in\Delta_\mathcal{A}} \mathcal{C}^{\pi_s,\mathcal{P}_s}_{+}\right] \cap \left[\bigcup_{\pi_s\in\Delta_\mathcal{A}} \mathcal{C}^{\pi_s,\mathcal{P}_s}_{-}\right]\right]\right].
\end{equation}
On the line $K$, denote the direction of the ray $\{\mathbf{u} + t\mathbf{e}_{s_1} \mid t\le 0\}$ as negative and the opposite direction as positive.

First, we have
\begin{equation}
    K \cap \mathcal{H}_{+}^{\pi_s, P_s} = \left\{ \mathbf{u} + t\mathbf{e}_{s_1} \mid t\langle\mathbf{e}_{s_1}, L^{\pi_s,P_s}\rangle \le r^{\pi_s} - \langle \mathbf{u}, L^{\pi_s,P_s} \rangle \right\}
\end{equation}
For $s\neq s_1$, since $\langle\mathbf{e}_{s_1}, L^{\pi_s,P_s}\rangle \le 0$, the intersection $K \cap \mathcal{H}_{+}^{\pi_s, P_s}$ is either the line $K$ or a negative ray. Thus, the intersection
\begin{equation}
  K \cap \left[ \bigcap_{s\in\mathcal{S}, s\neq s_1} \left[\bigcup_{\pi_s\in\Delta_\mathcal{A}} \mathcal{C}^{\pi_s,\mathcal{P}_s}_{+}\right] \right]
\end{equation}
is either the line $K$ or a negative ray.

For $s=s_1$, since $\langle\mathbf{e}_{s_1}, L^{\pi_s,P_s}\rangle > 0$, then the intersection $K \cap \mathcal{H}_{+}^{\pi_s, P_s}$ is a positive ray. Thus, the intersection
\begin{equation}
  K \cap \left[\bigcup_{\pi_s\in\Delta_\mathcal{A}} \mathcal{C}^{\pi_s,\mathcal{P}_s}_{+}\right]
\end{equation}
is also a positive ray.

Putting it together, we can obtain that the intersection
\begin{equation}
  K \cap \left[ \bigcap_{s\in\mathcal{S}} \left[\bigcup_{\pi_s\in\Delta_\mathcal{A}} \mathcal{C}^{\pi_s,\mathcal{P}_s}_{+}\right] \right]
\end{equation}
is either empty or a line segment (or a point in degenerated case).

Similarly, we can show that the intersection
\begin{equation}
  K \cap \left[ \bigcap_{s\in\mathcal{S}} \left[\bigcup_{\pi_s\in\Delta_\mathcal{A}} \mathcal{C}^{\pi_s,\mathcal{P}_s}_{-}\right] \right]
\end{equation}
is either empty or a line segment (or a point in degenerated case).

Finally, taking the intersection, we have that the intersection between $K$ and the robust value space is either empty or a line segment (or a point in degenerated case).
\end{proof}

\clearpage
\section{Additional Figures}
\label{app:figure}

\begin{figure}[ht!]
\centering
    \includegraphics[width=0.85\linewidth]{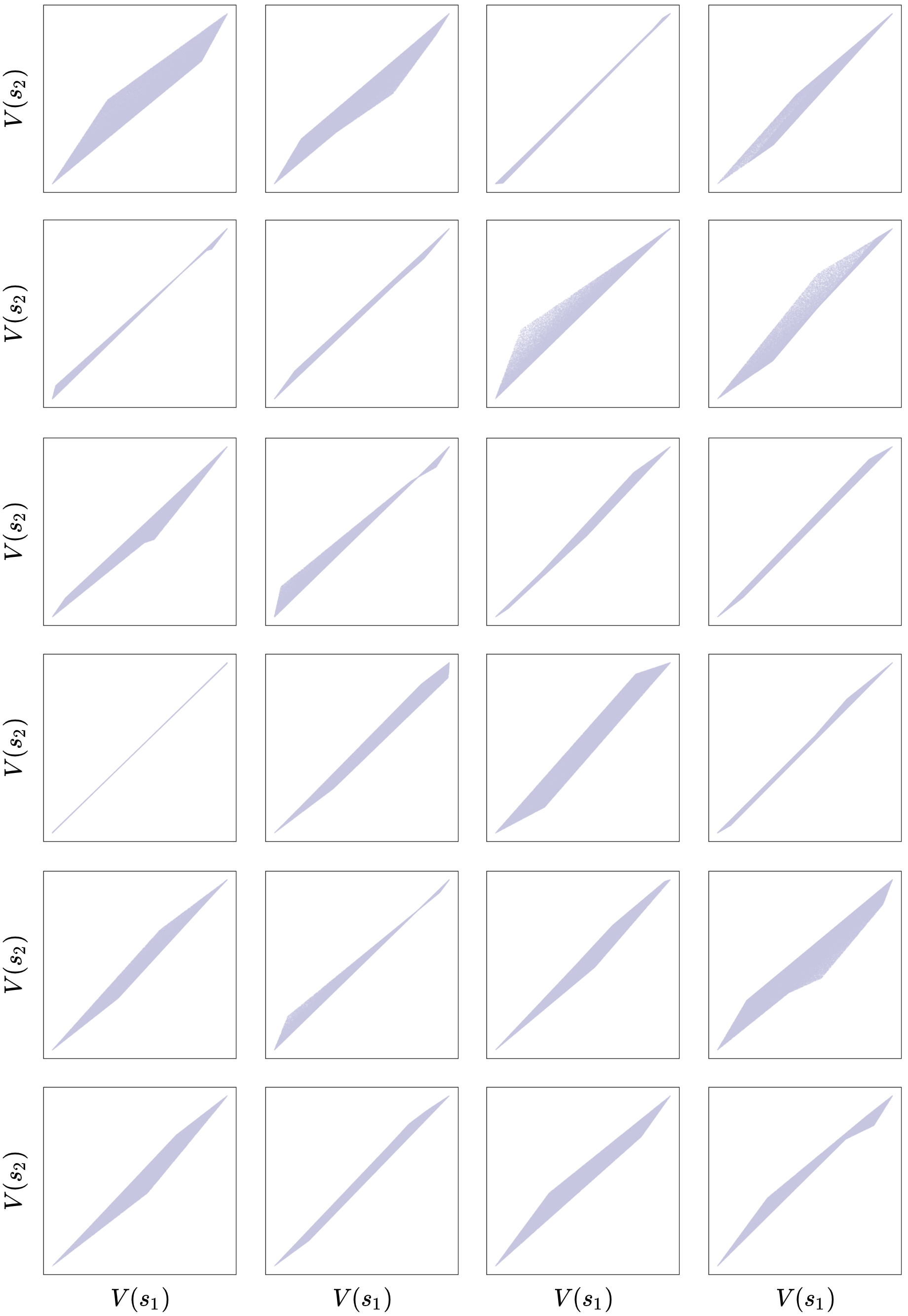}
    \caption{Visualization of the robust value space for several randomly generated $s$-rectangular RMDPs.}
     \label{fig:star-convexity}
\end{figure}


\end{document}